\documentclass{article}
\usepackage[T1]{fontenc}

\usepackage{iclr2027_conference,times}
\iclrfinalcopy

\usepackage{amsmath,amsfonts,bm}

\def\eqref#1{equation~\ref{#1}}

\def\1{\bm{1}}

\DeclareMathAlphabet{\mathsfit}{\encodingdefault}{\sfdefault}{m}{sl}
\SetMathAlphabet{\mathsfit}{bold}{\encodingdefault}{\sfdefault}{bx}{n}

\newcommand{\E}{\mathbb{E}}

\newcommand{\R}{\mathbb{R}}

\usepackage{amsmath,amssymb,amsthm}
\usepackage{booktabs}
\usepackage{longtable}
\usepackage{float}
\usepackage{graphicx}
\usepackage{subcaption}
\usepackage{xcolor}
\usepackage{microtype}
\usepackage{url}
\usepackage[colorlinks=true,citecolor=blue,linkcolor=blue,urlcolor=blue]{hyperref}
\usepackage{tikz}
\usepackage{algorithm}
\usepackage{algorithmic}


\newtheorem{proposition}{Proposition}

\newcommand{\cN}{\mathcal{N}}
\providecommand{\R}{\mathbb{R}}
\providecommand{\E}{\mathbb{E}}
\newcommand{\nmax}{n_{\mathrm{max}}}

\title{Sample-Size Extrapolation in Simulation-Based Inference via Statistical Scaling}

\author{Shiyi Sun$^{1}$ \qquad Lei Fang$^{2}$ \qquad Geoff K. Nicholls$^{1}$}

\begin{document}
\maketitle
{\renewcommand{\thefootnote}{}\footnotetext{$^{1}$Department of Statistics, University of Oxford, Oxford, United Kingdom.
$^{2}$Mathematical Institute, University of Oxford, Oxford, United Kingdom.
Correspondence to: Shiyi Sun \textless{}shiyi.sun@spc.ox.ac.uk\textgreater{}.}}
\setcounter{footnote}{0}

\begin{abstract}
In simulation-based inference (SBI), neural posterior estimators trained on a bounded range of sample sizes can become inaccurate on larger datasets because
their parameterizations do not specify how posterior location, uncertainty and shape must change as observations accumulate. We propose center-estimated NPE (CE-NPE), which combines a statistical center with an explicit contraction rate and learns a local center correction and residual densities in posterior-scale coordinates. The center is a statistical estimate, such as a closed-form statistic or a simulator-based moment match, which does not require an evaluable likelihood. Our analysis explains how contraction amplifies center error and decomposes Gaussian posterior error into shape, mean and covariance contributions. On six regular benchmarks, CE-NPE gives the best reference-posterior agreement of all tested amortized estimators at every out-of-range size, up to five times the largest training size, without retraining; it also fits better within the training range on most of them. Matched controls and nonregular examples characterize the scope of these gains and identify the conditions on which they depend: an appropriate contraction rate, an accurate center, and a stable residual representation.
\end{abstract}

\section{Introduction}

Simulation-based inference (SBI) targets models whose likelihood cannot be evaluated but which can be simulated \citep{cranmer2020frontier}. Neural posterior estimation (NPE), its most direct form, amortizes Bayesian inference by learning a conditional density from simulations and reusing it on new datasets \citep{papamakarios2016fast,greenberg2019automatic,lueckmann2021benchmarking}. In deployment, however, the estimator may be asked to condition on \(n\) observations when \(n\) exceeds every sample size seen during training. Permutation-invariant encoders such as Deep Sets accommodate this change in cardinality \citep{zaheer2017deep}, but the posterior must also change its location, width and shape as evidence accumulates, and providing the network with $n$ as an input \citep{radev2020bayesflow,habermann2025amortized,mittal2025amortized} leaves these relationships to be learned over a bounded training range. Accepting arbitrary dataset sizes is therefore an architectural property, whereas extrapolating posterior uncertainty is a statistical one. We ask \textbf{how posterior estimators trained on small datasets can extend to larger sample sizes without retraining}.

Statistical asymptotics describe how a posterior should change with $n$, and therefore supply structure for this problem. Under regularity conditions, the Bernstein--von Mises theorem describes an approximately Gaussian posterior with width of order $n^{-1/2}$ after appropriate centering \citep[Ch.~10]{vandervaart1998asymptotic}. More generally, a justified contraction rate $n^{-\alpha}$, which nonregular models can change \citep{Bochkina2014}, suggests learning in local coordinates $n^\alpha(\theta-c_n)$. Working in these coordinates removes the need for the network to infer the leading contraction rate from the training range, but it makes center accuracy critical: a location error $\delta_n$ becomes a shift $n^\alpha\delta_n$ in those coordinates, so a small fixed bias can become large relative to posterior uncertainty. Moreover, even a center that contracts at the correct rate leaves errors in covariance and residual shape, so scaling alone does not ensure calibrated inference.

We propose center-estimated NPE (CE-NPE), which anchors the posterior at a statistical estimator $c_{\mathrm{stat}}(x_{1:n})$ and learns only a correction, multiplied by $n^{-\alpha}$. This design places the learned adjustment on the posterior's contraction scale and thereby reduces what must be extrapolated in parameter space; because it supplies structure that would otherwise have to be learned from simulations, it can also improve fitting within the training range. The anchor can be a closed-form statistic or a simulator-based moment estimate, so it does not require an evaluable likelihood, and the network still learns uncertainty, dependence and residual shape. Figure~\ref{fig:concept} relates CE-NPE to free conditioning and to structural scaling with a learned center, which serves as an intermediate ablation.

Our experiments examine both the benefit of this parameterization and its source: six adapted regular benchmarks test reference-posterior agreement beyond the training range, matched-center controls assess what scaling adds beyond the statistical estimator, best-effort baselines assess sensitivity to training effort, and nonregular examples test the assumptions behind the local representation. Throughout, we report in-range performance alongside extrapolation results, separating errors already present within the training range from those that arise beyond it.

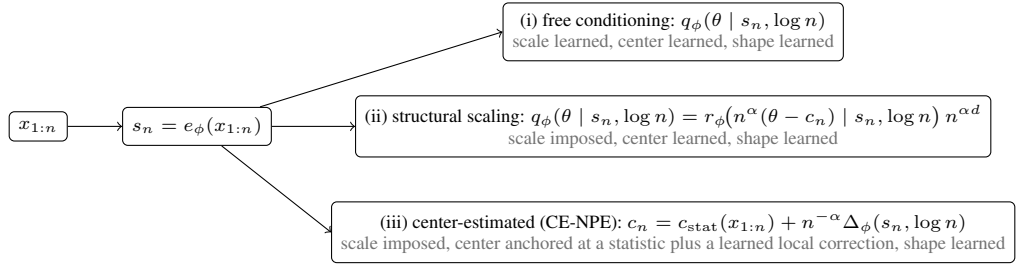
\begin{figure}[t]
\centering
\begin{tikzpicture}[font=\scriptsize,
  box/.style={draw, rounded corners=2pt, align=center, inner sep=3.5pt}]
\node[box] (data) at (0.5,0) {$x_{1:n}$};
\node[box] (enc) at (2.6,0) {$s_n=e_\phi(x_{1:n})$};
\draw[->] (data) -- (enc);
\node[box] (free)  at (8.9,1.25) {(i)~free conditioning: $q_\phi(\theta\mid s_n,\log n)$\\
  \textcolor{black!55}{\scriptsize scale learned, center learned, shape learned}};
\node[box] (struc) at (8.9,0) {(ii)~structural scaling: $q_\phi(\theta\mid s_n,\log n)=r_\phi\!\big(n^{\alpha}(\theta-c_n)\mid s_n,\log n\big)\,n^{\alpha d}$\\
  \textcolor{black!55}{\scriptsize scale imposed, center learned, shape learned}};
\node[box] (ctx)   at (8.9,-1.4) {(iii)~center-estimated (CE-NPE): $c_n=c_{\mathrm{stat}}(x_{1:n})+n^{-\alpha}\Delta_\phi(s_n,\log n)$\\
  \textcolor{black!55}{\scriptsize scale imposed, center anchored at a statistic plus a learned local correction, shape learned}};
\draw[->] (enc) -- (free.west);
\draw[->] (enc) -- (struc.west);
\draw[->] (enc) -- (ctx.west);
\end{tikzpicture}
\caption{\textbf{Progressively removing what the network must extrapolate.}  All three constructions share a permutation-invariant encoder and differ in which ingredients of the local coordinates $z_n=n^{\alpha}(\theta-c_n)$ are learned from the bounded training range and which are imposed: (i) free conditioning learns scale, center, and shape; (ii) structural scaling imposes the scale; (iii) CE-NPE additionally anchors the center at a statistical estimator. }
\label{fig:concept}
\end{figure}

\paragraph{Contributions.}
\begin{itemize}
\item \textbf{A structured parameterization}. We introduce CE-NPE, which combines a statistical center with an explicit contraction rate and learns the center correction and residual density
in local coordinates. The center need not evaluate the likelihood and may be constructed from simple statistics or simulator access, so the estimator remains likelihood-free, and the parameterization reduces what the network must extrapolate beyond the training sample sizes (Section~\ref{sec:methods}). The ingredients are classical; the contribution is to place them inside an amortized estimator and to show that its out-of-range gains are attributable to this structure rather than to network capacity.
\item \textbf{Explaining the errors that remain after scaling.}
We show how posterior contraction amplifies center error and separate Gaussian posterior error into shape, mean and covariance contributions. These results explain why correct scaling and a contracting center alone do not guarantee accurate, calibrated inference (Section~\ref{sec:anatomy}).
\item \textbf{Testing the benefits and limits of structural extrapolation.}
We test whether explicit statistical structure improves amortized posterior estimation beyond the training sample-size range, under matched architectures and simulation budgets. On six regular benchmarks \citep{lueckmann2021benchmarking}, center anchoring gives the best reference agreement at every out-of-range size against free conditioning, structural scaling with a learned center, and a compositional score-based estimator \citep{geffner2023compositional}. Fixed-center controls isolate the contribution of scaling, while best-effort baselines assess sensitivity to training effort. Four nonregular prototypes with a faster rate, a boundary, a ridge and persistent multimodality show where the imposed structure fails, and a real-data application to 1000 Genomes windows \citep{auton2015global} assesses consistency and held-out prediction without a reference posterior (Sections~\ref{sec:experiments} and~\ref{sec:discussion}; Appendices~\ref{app:regular}--\ref{app:popgen}).
\end{itemize}

\section{Related work}

\paragraph{Neural simulation-based inference.}
Neural SBI learns posterior densities, likelihoods, or likelihood ratios from simulations \citep{papamakarios2016fast,greenberg2019automatic,papamakarios2019sequential,hermans2020likelihoodfree}; see \citet{cranmer2020frontier} for a survey. Neural posterior estimation fits a conditional density estimator to simulated parameter--data pairs, either sequentially around one observation \citep{greenberg2019automatic,deistler2022truncated} or amortized over the prior, and its density family has moved from normalizing flows \citep{durkan2019neural} to diffusion heads \citep{chen2025conditional} and to transformers that serve many conditional queries with one model \citep{gloeckler2024allinone}. Amortization has also been extended to misspecified simulators \citep{ward2022robust} and beyond the standard Bayesian target, to generalized Bayesian posteriors indexed by a loss temperature \citep{sun2026amortized}. Benchmarks \citep{lueckmann2021benchmarking} and calibration studies \citep{hermans2022trust} evaluate these estimators at fixed sample sizes. We study a different axis: how an amortized estimator trained on a bounded range of sample sizes transfers to larger datasets, and we introduce explicit statistical structure into its dependence on sample size.

\paragraph{Inference with variable-size datasets.}
Deep Sets and Set Transformer provide permutation-invariant representations of variable-size inputs \citep{zaheer2017deep,lee2019set}, and \citet{chan2018likelihood} apply exchangeable networks to likelihood-free inference from population genetic data. BayesFlow combines learned summaries with amortized posterior estimation, trains across variable numbers of observations, and shows posterior contraction empirically as observations accumulate \citep{radev2020bayesflow}. More recent amortized estimators condition transformers on sets of observations of varying size, within multilevel models \citep{habermann2025amortized} and as in-context posterior estimation \citep{mittal2025amortized}. These architectures accommodate changing dataset size, but do not themselves prescribe a contraction rate. CE-NPE retains a set encoder and adds a specified rate and a statistical center to the posterior parameterization.

\paragraph{Compositional inference.}
For conditionally independent observations, learned likelihoods or likelihood ratios can be combined across observations to perform inference at different sample sizes \citep{papamakarios2019sequential,hermans2020likelihoodfree}. Compositional score methods similarly construct posterior samplers from networks trained on individual observations or small subsets \citep{geffner2023compositional,linhart2026diffusion}. These methods exploit factorization to accommodate additional observations, but require the likelihood to factorize across subsets. CE-NPE instead models the full-data posterior in centered, rescaled coordinates, relying on an accurate statistical anchor and a stable residual representation rather than on a factorization (comparison with PF-NPSE in Appendix~\ref{app:pfnpse}).

\paragraph{Regression adjustments and learned summaries.}
ABC regression adjustment modifies simulated parameter draws to account for discrepancies between simulated and observed summaries \citep{beaumont2002approximate}, and semi-automatic ABC uses regression estimates of posterior means as summary statistics \citep{fearnhead2012constructing}. CE-NPE uses an estimator for a different purpose, anchoring the coordinates in which an amortized posterior is learned, with a learned correction expressed in local coordinates and mapped to parameter space through $n^{-\alpha}$.

\paragraph{Asymptotic structure in posterior estimation.}
Under regularity conditions, the Bernstein--von Mises theorem motivates centering and rescaling a posterior by $\sqrt{n}$ \citep[Ch.~10]{vandervaart1998asymptotic}. Asymptotic reasoning also informs parametric choices for amortized posterior approximation \citep{maceda2026variational}, while \citet{frazier2024statistical} study the statistical accuracy of neural posterior and likelihood estimation. Our focus is transfer from a bounded training range in sample size: a justified contraction rate defines local coordinates, and we analyze the remaining errors from centering, covariance and residual shape.

\section{Methods}
\label{sec:methods}

Let $\theta\in\Theta\subseteq\mathbb{R}^d$ have prior $p(\theta)$, and let $x_{1:n}\overset{\mathrm{i.i.d.}}{\sim}p(\cdot\mid\theta),$ where simulation from the generative model is available although the likelihood need not be evaluable. We train amortized estimators on simulated datasets with $n\leq n_{\max}$ and evaluate them at sample sizes $n>n_{\max}$. Every estimator is trained with the standard NPE objective, the log density assigned to the simulated parameter given the simulated dataset and its size \citep{papamakarios2016fast,greenberg2019automatic}. Because the observations are exchangeable, all estimators share the same permutation-invariant Deep Sets encoder $s_n=e_\phi(x_{1:n})$ \citep{zaheer2017deep,radev2020bayesflow}; the comparison focuses on how the posterior parameterization depends on $n$, given a representation that already accepts variable-size datasets. The three constructions of Figure~\ref{fig:concept} differ only in how much of that dependence is learned and how much is imposed.

\subsection{Free conditioning}
\label{sec:free}

The most direct approach is to provide sample size as an additional conditioning variable:
\begin{equation}
q_\phi^{\mathrm{free}}(\theta\mid x_{1:n},n)
=
q_\phi(\theta\mid s_n,\log n).
\label{eq:free}
\end{equation}
The logarithm only places the sample-size input on a convenient numerical scale. Apart from this explicit input, the construction is the standard one when a set encoder is combined with amortized posterior estimation \citep{zaheer2017deep,radev2020bayesflow}. It imposes no constraint on how posterior location, scale, or shape evolve with $n$: any contraction must be inferred from the values of $n$ seen in training, allowing flexible interpolation within the training range but requiring the learned dependence to be extended beyond it.

\subsection{Structural scaling}
\label{sec:structural}

In regular parametric problems the Bernstein--von Mises theorem implies that posterior uncertainty shrinks at the leading rate $n^{-1/2}$ \citep[Ch.~10]{vandervaart1998asymptotic}. More generally, suppose the relevant contraction rate is $n^{-\alpha}$, with $\alpha=1/2$ in the regular case, and define the rescaled variable
\begin{equation}
z_n
=
n^\alpha(\theta-c_n),
\label{eq:residual}
\end{equation}
where $c_n$ is a data-dependent center. In the basic structural estimator the center is learned from the encoded dataset, $c_n = c_\phi(s_n,\log n)$. Rather than modelling $\theta$ directly, we learn a conditional density $r_\phi$ for $z_n$:
\begin{equation}
z_n\mid x_{1:n}
\sim
r_\phi(\,\cdot\mid s_n,\log n).
\end{equation}
The corresponding density in the original parameterization is
\begin{equation}
q_\phi^{\mathrm{scale}}(\theta\mid x_{1:n},n)
=
r_\phi\!\left(
n^\alpha(\theta-c_n)\mid s_n,\log n
\right)n^{\alpha d}.
\label{eq:induced}
\end{equation}
Sample size therefore enters through the coordinate transformation and its Jacobian, as well as through the $\log n$ input to the network. The transformation imposes a leading contraction factor, while the network learns the center and residual distribution. Because a size-dependent residual can still change its own width, the explicit factor alone does not fix the fitted contraction rate.

\subsection{Center-estimated NPE}
\label{sec:ce}

With the rate imposed, the center $c_n$ must still extrapolate, and in the rescaled coordinates a center error $\delta_n$ becomes $n^\alpha\delta_n$ (Section~\ref{sec:anatomy}). CE-NPE therefore anchors the center at a problem-specific estimator $c_{\mathrm{stat}}(x_{1:n})$ and lets the network $\Delta_\phi(s_n,\log n)$ learn only a correction:
\begin{equation}
c_n
=
c_{\mathrm{stat}}(x_{1:n})
+
n^{-\alpha}\Delta_\phi(s_n,\log n).
\label{eq:ce-center}
\end{equation}
The factor $n^{-\alpha}$ expresses the learned correction in the same local scale as the posterior uncertainty; Proposition~\ref{prop:center-coordinates} (Appendix~\ref{app:center-theory}) shows that with this factor a bounded correction preserves the anchor's error order, which a correction in parameter units does not. Substituting this center into $z_n=n^\alpha(\theta-c_n)$ gives
\begin{equation}
z_n
=
n^\alpha\!\left[
\theta-c_{\mathrm{stat}}(x_{1:n})
\right]
-
\Delta_\phi(s_n,\log n).
\end{equation}

The center is only a coordinate anchor: it need not be a MAP estimate, require an evaluable likelihood, or approximate posterior uncertainty or shape. Its key requirement is to track posterior location at the contraction rate, with error $O_p(n^{-\alpha})$, so that the error remains $O(1)$ in rescaled coordinates (Section~\ref{sec:anatomy}). In our benchmarks, $c_{\mathrm{stat}}$ ranges from simple closed-form statistics such as $\bar x$ to EM or simulator-based moment matching (Appendix~\ref{app:benchmarks}).

\subsection{Residual density families}
\label{sec:residual-family}

The structured parameterization does not require the rescaled posterior to be Gaussian, and we consider two residual families. The first is a Gaussian conditional density,
\begin{equation}
r_\phi(z_n\mid s_n,\log n)
=
\mathcal{N}\!\left(
z_n;
\mu_\phi(s_n,\log n),
\Sigma_\phi(s_n,\log n)
\right),
\end{equation}
which provides a simple residual family and lets us assess how far Gaussianity suffices after centering and scaling. This is not Gaussian variational inference: mean and covariance are amortized functions of $(s_n,\log n)$, and in parameter space the density has center $c_n+n^{-\alpha}\mu_\phi$ and covariance $n^{-2\alpha}\Sigma_\phi$. The second family uses a neural spline flow (NSF) \citep{durkan2019neural}, which can represent
non-Gaussian residual shapes. The free heads and the CE residual receive $\log n$ as a size input, so the
conditional residual can account for finite-sample changes in shape and covariance, while the deterministic factor $n^{-\alpha}$ sets the contraction scale; the size inputs of the saved runs are documented in Appendix~\ref{app:size-context}.

\section{Why scaling alone is not enough}
\label{sec:anatomy}

The factor $n^{-\alpha}$ fixes only the reference contraction scale; accurate extrapolation also requires that posterior location, uncertainty and shape remain correct as $n$ grows.

\subsection{Keeping center error on the posterior scale}
Let $m_n$ be the true posterior mean and $c_n$ the center of the residual coordinates. If posterior width decreases as $n^{-\alpha}$, a location error $c_n-m_n$ becomes a shift $n^\alpha(c_n-m_n)$ on the posterior scale: the same absolute error becomes five times larger in posterior-standard-deviation units when the posterior becomes five times narrower (Figure~\ref{fig:amplify}). The residual decomposes as
\begin{equation}
\underbrace{n^\alpha(\theta-c_n)}_{\text{residual to be learned}}
=\underbrace{n^\alpha(\theta-m_n)}_{\text{variation around the true mean}}
+\underbrace{n^\alpha(m_n-c_n)}_{\text{shift caused by center error}},
\label{eq:decomp}
\end{equation}
Keeping this shift bounded requires the center error to decrease at least as fast as the posterior width. Otherwise, the residual network must learn a location shift that grows with sample size, even after the contraction scale has been removed.

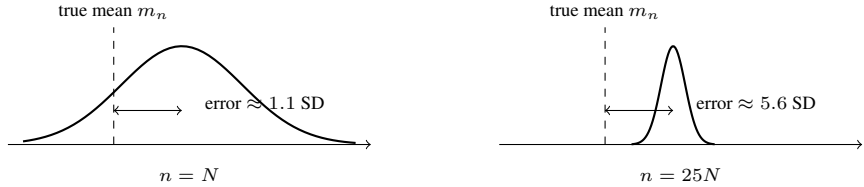
\begin{figure}[h]
\centering
\begin{tikzpicture}[font=\scriptsize]
\begin{scope}
\draw[->] (0,0) -- (4.8,0);
\draw[dashed] (1.4,0) -- (1.4,1.55);
\node at (1.4,1.75) {true mean $m_n$};
\draw[thick] plot[domain=0.2:4.6,samples=60]
(\x,{1.3*exp(-((\x-2.3)^2)/1.28)});
\draw[<->] (1.4,0.45) -- (2.3,0.45);
\node at (3.4,0.55) {error $\approx1.1$ SD};
\node at (2.4,-0.4) {$n=N$};
\end{scope}
\begin{scope}[shift={(6.5,0)}]
\draw[->] (0,0) -- (4.8,0);
\draw[dashed] (1.4,0) -- (1.4,1.55);
\node at (1.4,1.75) {true mean $m_n$};
\draw[thick] plot[domain=1.75:2.85,samples=60]
(\x,{1.3*exp(-((\x-2.3)^2)/0.0512)});
\draw[<->] (1.4,0.45) -- (2.3,0.45);
\node at (3.4,0.55) {error $\approx5.6$ SD};
\node at (2.4,-0.4) {$n=25N$};
\end{scope}
\end{tikzpicture}
\caption{\textbf{The same location error becomes more serious as uncertainty shrinks.}}
\label{fig:amplify}
\end{figure}

CE-NPE uses $c_n=c_{\mathrm{stat}}+n^{-\alpha}\Delta_n$ to express the learned correction on this same scale (Equation~\ref{eq:ce-center}). Proposition~\ref{prop:center-coordinates} (Appendix~\ref{app:center-theory}) shows that, for a fixed trained network, if $c_{\mathrm{stat}}-m_n=O_p(n^{-\alpha})$ and $\Delta_n=O_p(1)$ over repeated datasets, then $c_n-m_n=O_p(n^{-\alpha})$. A bounded local correction therefore preserves the anchor's error order; a correction added directly in parameter units must itself shrink at rate $n^{-\alpha}$ or faster to provide the same guarantee. Neither condition is enforced by training, and a bounded local correction cannot remove a persistent anchor bias; the anchor itself must therefore contract.

\subsection{Separating location, covariance, and shape errors}
\label{sec:hierarchy}
Keeping the shift bounded is necessary but not sufficient: an error that remains bounded relative to posterior width need not be small. For example, two one-dimensional Gaussians with equal variances and means separated by one standard deviation have forward KL divergence $1/2$, even as both variances shrink.

Let $P_n$ denote the true posterior, $Q_n$ a fitted Gaussian posterior, and $G_n$ the Gaussian with the same mean and covariance as $P_n$. Under the density, entropy, and covariance conditions stated in Proposition~\ref{prop:finite-error} (Appendix~\ref{app:transfer}),
\begin{equation}
\mathrm{KL}(P_n\Vert Q_n)
=\underbrace{\mathrm{KL}(P_n\Vert G_n)}_{\text{non-Gaussian shape}}
+\underbrace{\mathrm{KL}(G_n\Vert Q_n)}_{\text{mean and covariance mismatch}}.
\label{eq:main-error}
\end{equation}
The first term measures the error that remains even for the best Gaussian approximation (Proposition~\ref{prop:finite-error}). The second measures errors in the fitted mean, marginal widths, and correlations. Because the learned residual can have a nonzero mean, the fitted mean need not equal $c_n$; controlling the center addresses only part of posterior accuracy.

For a fixed dataset, free conditioning, structural scaling, and statistical anchoring do not change the best achievable approximation within an unrestricted Gaussian family; they change how its parameters are learned and extrapolated. A spline residual can reduce the Gaussian shape error, but greater flexibility does not guarantee that training achieves this. The experiments therefore assess posterior-mean error, marginal width, and joint posterior agreement separately.

Both analyses hold at any fixed sample size; whether the fitted correction, covariance and residual shape remain accurate beyond the training range is an empirical question, which Section~\ref{sec:experiments} addresses by reporting fit at the training boundary and beyond it separately.

\section{Experiments}
\label{sec:experiments}

\subsection{Experimental setup}
\label{sec:protocol}

\paragraph{Estimators and protocol.}
We compare five estimators under matched encoder widths, optimization schedules, and training simulation budgets: \textbf{Free-G} and \textbf{Free-NSF}, which learn the posterior through free conditioning; \textbf{Scale-G}, which imposes the contraction scale but learns the center; and \textbf{CE-G} and \textbf{CE-NSF}, which combine the imposed scale with a statistical anchor and a learned local correction. The suffixes G and NSF denote Gaussian and spline density heads, respectively. Training sizes are sampled uniformly from $n\in\{5,\ldots,20\}$, and evaluation uses $n\in\{10,20,50,100\}$, with datasets paired across estimators and seeds. Out-of-range sizes are used only for evaluation. All regular benchmarks impose the regular rate $\alpha=1/2$. When no asymptotic argument supplies the rate, it can be estimated from how the anchor contracts across the training sizes; on our tasks this estimate serves as a check on the assumed rate rather than a substitute for it (Appendix~\ref{app:alpha-hat}). Each training step draws a fresh batch of $128$ simulated datasets for up to $4000$ steps, evaluation uses $400$ paired datasets per size, and reported values are means over three seeds (five on the mixture). Architectural and implementation details are given in Appendix~\ref{app:details}.

\paragraph{Tasks.}
\label{sec:benchmark-design}
We use six adapted regular benchmarks from the \texttt{sbibm} suite \citep{lueckmann2021benchmarking}: Gaussian linear with Gaussian and uniform priors, Gaussian mixture, Bernoulli GLM, SIR, and Lotka--Volterra. Each dataset contains $n$ independent simulator replicates. The SIR and Lotka--Volterra priors are truncated to a regular regime in which the data remain informative about every parameter. Benchmark definitions, centers, and reference computations are given in Appendix~\ref{app:benchmarks}. Symmetric multimodal benchmarks and nonregular stress tests are reported in Appendix~\ref{app:boundary}.

\paragraph{Evaluation metrics.}
We assess posterior location, marginal width, and joint agreement with exact or numerically computed reference posteriors. Center error measures the fitted posterior-mean error in reference-standard-deviation units, averaged over coordinates; the SD ratio targets $1$. Joint $95\%$ HPD coverage \citep{hermans2022trust} targets $0.95$, and C2ST \citep{lopez2017revisiting} targets $0.5$. These complementary diagnostics distinguish improvements in location and contraction from remaining errors in joint posterior structure (Section~\ref{sec:hierarchy}). KL divergence, MMD \citep{gretton2012kernel} and simulation-based calibration \citep{talts2018validating} appear in the appendix studies that use them.

\begin{figure}[t]
\centering
\includegraphics[width=0.85\linewidth]{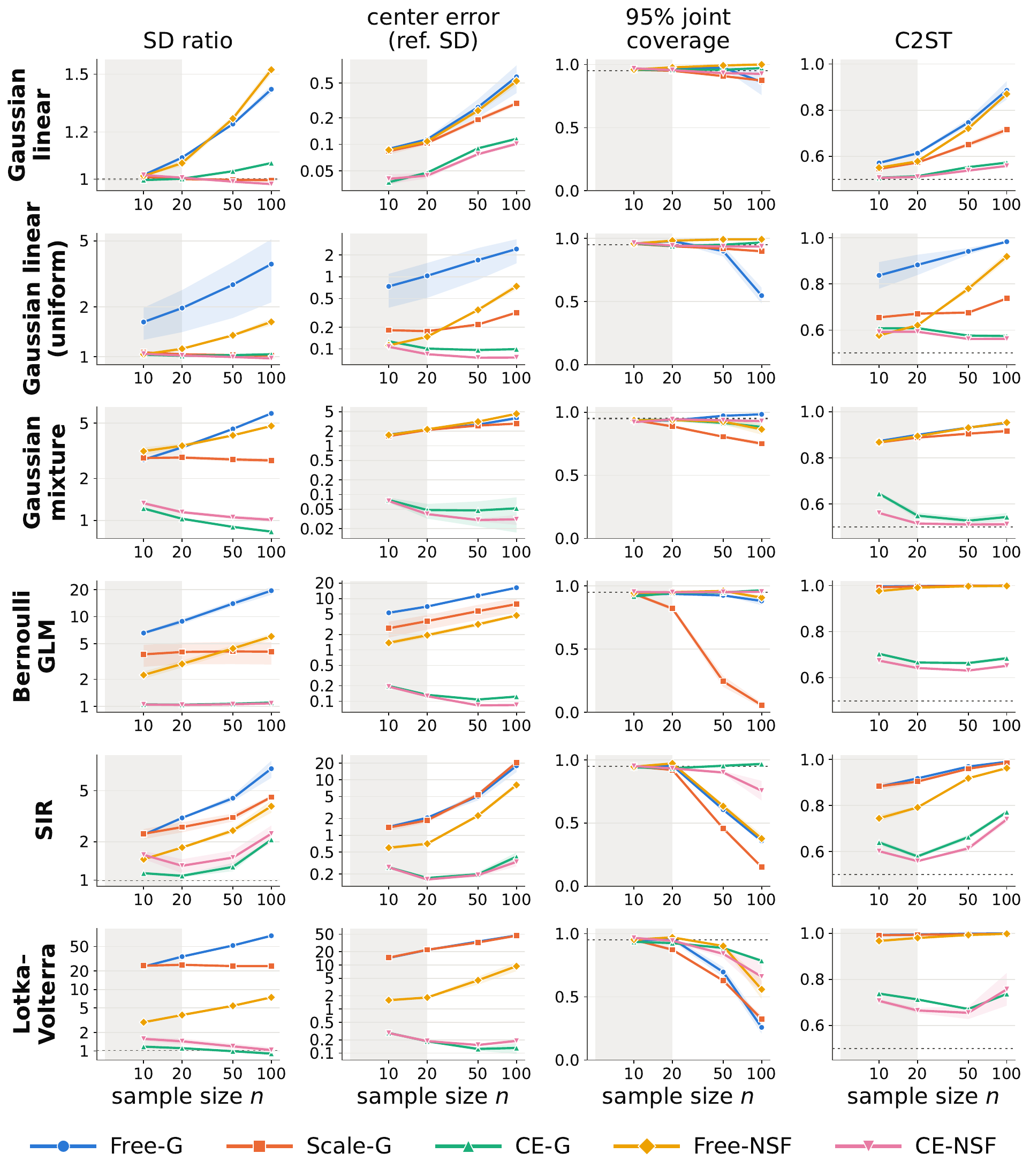}
\caption{\textbf{Results on the six regular benchmarks.} Seed means with $\pm1$ seed-SD bands. The shaded band is the training range; dotted lines mark the targets ($1$, $0.95$, and $0.5$). SD ratio and center error are on log axes. }
\label{fig:sbibm}
\end{figure}

\subsection{Does statistical structure improve sample-size extrapolation?}
\label{sec:gmm-results}

\paragraph{Overall performance.}
CE-G and CE-NSF achieve the lowest C2ST at both out-of-range sizes on all six regular benchmarks (Figure~\ref{fig:sbibm}; full results at every size in Appendix~\ref{app:regular}), and they also improve in-range agreement on most tasks. The in-range gain is expected: at $n$ between $5$ and $20$ the posterior is already far narrower than the prior on these benchmarks, so the anchor and the imposed scale supply the location and width that a free network must otherwise learn. The advantage does not, however, extend uniformly to calibration: at $n=100$, both CE estimators remain over-dispersed on SIR, where only CE-G maintains near-nominal joint coverage, and on Lotka--Volterra their marginal widths are close to the reference yet both under-cover.
\paragraph{Roles of scaling and centering.}
The intermediate ablation separates the two ingredients. Scale-G keeps its relative width approximately constant on five benchmarks and slows its growth on SIR, but because its center is still learned, its posterior-mean error in reference SD units increases with sample size on every benchmark. Anchoring the center removes this growth: CE-G improves both center error and C2ST over the Gaussian baselines at every out-of-range size. A fixed-center mixture control isolates the benefit of scaling more directly: with the same EM posterior mean, imposing the contraction factor substantially reduces C2ST and KL at $n=100$, although coverage moves farther from nominal (Appendix~\ref{app:matched}). This control also outperforms full CE-G in C2ST and KL, so on the mixture the gain comes from explicit scaling around an accurate center, and the learned location adjustment adds no further improvement. The same controls on SIR and Lotka--Volterra lead to the same conclusion: the fixed-center, imposed-scale control matches or slightly improves on CE-G at every size, and the in-range advantage of CE-NSF narrows or reverses at $5\times$ (Appendix~\ref{app:matched}).

\paragraph{Residual flexibility.}
The spline head adds a further, task-dependent benefit: CE-NSF improves on CE-G in width and C2ST on the mixture and matches it on the other benchmarks, whereas Free-NSF, which lacks the centered coordinates, remains inaccurate. Flexibility therefore helps once the coordinates are right, but does not substitute for them.

\subsection{Real data: mutation rate and growth from 1000 Genomes}
\label{sec:real}
\begin{figure}[b]
\centering
\includegraphics[width=\linewidth]{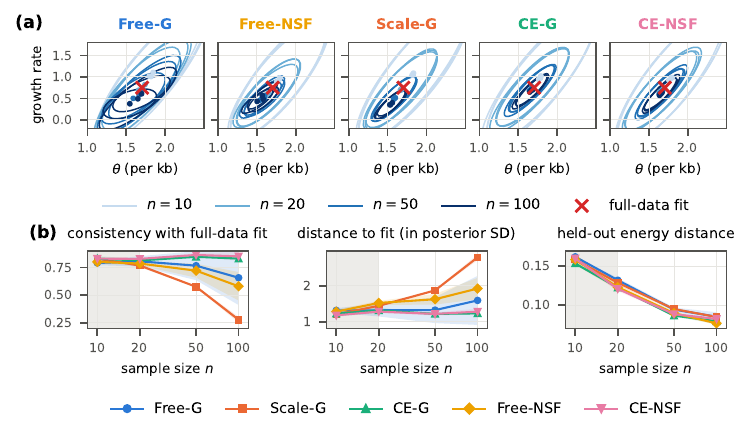}
\caption{\textbf{Posterior contraction on real 1000 Genomes windows.} (a) The typical posterior of each estimator as $n$ grows from $10$ to $100$, with the moment fit to all windows as the cross. (b) Against $n$: \emph{consistency}, the fraction of datasets whose $95\%$ intervals contain the full-data fit (higher is better); \emph{distance}, how far the fit lies from the posterior mean in posterior SDs (lower is better); \emph{energy distance} between posterior-predictive and held-out real windows (lower is better; independent of the anchor). Tables in Appendix~\ref{app:popgen}.}
\label{fig:popgen}
\end{figure}

The regular benchmarks all have reference posteriors; population genetics provides a real-data application without one. Under a neutral coalescent model with growth, we infer the per-kb mutation rate and growth rate from genomic-window summaries, using a simulator-based moment estimator as the statistical anchor. The data are summaries of non-overlapping $10$\,kb windows on chromosome 22 of the 1000 Genomes phase-3 panel, and a dataset of size $n$ is a subsample of $n$ windows. The five estimators follow the main training-size protocol.

With no reference posterior, we use two complementary criteria (Figure~\ref{fig:popgen}): consistency with the moment fit to all windows, which shares the center's moment functional and therefore favors CE-NPE, and posterior-predictive energy distance to held-out windows, which is independent of that target.

As $n$ grows to five times the training maximum, the posteriors of CE-G and CE-NSF contract onto the full-data moment fit, whereas Scale-G and Free-NSF contract onto a lower growth rate and Free-G barely contracts (Figure~\ref{fig:popgen}a); accordingly, only the CE estimators keep their consistency with the fit while the others drift away from it (Figure~\ref{fig:popgen}b). This stability costs nothing in predictive accuracy: on the anchor-independent criterion, their held-out predictions match the best free baseline at every size. In a separate in-model evaluation, CE has the lowest SBC KS statistics, although its joint coverage remains below nominal (Appendix~\ref{app:popgen}).

\section{Discussion and limitations}
\label{sec:discussion}

\paragraph{Discussion.}
Taken together, the experiments support combining a statistical anchor with an explicit contraction scale: this structure reduces the sample-size dependence that a neural posterior estimator must learn, and it changes how posterior parameters are learned and extrapolated rather than enlarging the Gaussian approximation family (Section~\ref{sec:hierarchy}). The fixed-center controls clarify the source of these gains: with covariance still learned, anchoring at the statistic and imposing the scale match or improve on full CE-G in out-of-range C2ST on the mixture, SIR, and Lotka--Volterra (Appendix~\ref{app:matched}). CE-NPE is accordingly intended for datasets comprising independent simulator replicates conditional on the parameter, when a contraction rate is justified and an affordable statistical center tracks posterior location on that scale.

\paragraph{Alternative explanations.}
Because several free baselines already fit poorly within the training range, we ran three controls to check that the out-of-range gap is not an artifact of under-training, of the particular training maximum, or of the choice of comparator (Appendix~\ref{app:diagnostics}). First, best-effort free baselines, retrained on every benchmark with a tuned learning rate, three times the step budget and longer patience, narrow the in-range gap on four benchmarks but leave the out-of-range ordering unchanged on all six. Second, retraining every estimator on SIR with a four times larger training maximum preserves CE's C2ST advantage at $5N$, although absolute accuracy and coverage change. Third, the compositional comparator PF-NPSE, trained on the same number of simulated datasets, approaches CE-NSF on Gaussian linear but remains less accurate on the other regular benchmarks, with several failures already present in range. Neither training effort, training range nor composition therefore accounts for the gap.

\paragraph{Scope and failure modes.}
\label{sec:ablations}
\label{sec:stress}
\label{sec:boundary-results}
The nonregular prototypes probe the rate, centering, and residual assumptions (Appendix~\ref{app:boundary}). On T1, whose posterior contracts at rate $n^{-1}$ to a Gamma limit \citep{Bochkina2014}, imposing the regular exponent produces increasing width mismatch; correcting the exponent substantially improves agreement (Table~\ref{tab:b1-rate}). T2 instead exposes a learning failure at a boundary: neither fitted residual recovers the shape, although a Gaussian renormalized on the parameter support can represent the exact posterior. A persistent ridge or separated modes, as in T3 and T4, do not become a stable local posterior under contraction about a single center. The two symmetric \texttt{sbibm} benchmarks, SLCP and two moons, behave in the same way: every single-center estimator fails at every size, including within the training range (Appendix~\ref{app:boundary}); a flexible flow can represent such distributions but must learn their changing geometry in the rescaled coordinates.

\paragraph{Practical and empirical limitations.}
CE-NPE is only partially amortized: simulator-based anchors can require numerical optimization at each query, adding task-specific computation beyond the neural estimator (Appendix~\ref{app:benchmarks}). The experiments cover finite simulation budgets and extrapolation up to five times the training maximum; additional training on every benchmark and a second training maximum on SIR test this sensitivity (Appendix~\ref{app:diagnostics}). Better agreement than the tested baselines does not by itself ensure calibrated credible regions, as the regular-benchmark and in-model population-genetic results show (Appendices~\ref{app:regular} and~\ref{app:popgen}); calibration remains a separate concern in neural SBI \citep{hermans2022trust}.

\section{Conclusion}
Accepting arbitrary dataset sizes is architectural; extrapolating posterior uncertainty is statistical. Explicitly encoding a justified posterior contraction rate removes the leading scale-extrapolation burden, a contracting statistical center controls the amplified location error, and once these first-order effects are stabilized, residual density learning targets the remaining covariance and shape errors. On six regular benchmarks CE-NPE gives the best reference agreement of the tested amortized estimators at every out-of-range size, up to five times the largest training size, without retraining; on real genomic windows CE-NPE stays consistent with the full-data fit at the same relative sizes at no cost in held-out predictive accuracy. Where the rate or the center is not justified, as in the nonregular prototypes, the same analysis identifies which ingredient fails and points toward representations adapted to the posterior geometry.

\subsection*{AI use statement}

LLM-based assistants helped polish the language and assisted with experiment, analysis, and figure code. The authors verified all content and take responsibility for the manuscript, code, and results.

\subsection*{Reproducibility statement}

Task-specific protocols, reference calculations and per-case results are provided in the appendix and accompanying code (supplementary material). Every table can be regenerated from the included per-case files. Where an evaluation uses a subset of the paired cases, the appendix states the subset.

\bibliography{references}

@inproceedings{geffner2023compositional,
  title={Compositional Score Modeling for Simulation-Based Inference},
  author={Geffner, Tomas and Papamakarios, George and Mnih, Andriy},
  booktitle={Proceedings of the 40th International Conference on Machine Learning},
  pages={11098--11116}, year={2023}, volume={202}, series={PMLR}
}

@inproceedings{chen2025conditional,
  title={Conditional Diffusions for Amortized Neural Posterior Estimation},
  author={Chen, Tianyu and Bansal, Vansh and Scott, James G.},
  booktitle={Proceedings of the 28th International Conference on Artificial Intelligence and Statistics},
  pages={2377--2385}, year={2025}, volume={258}, series={PMLR}
}

@inproceedings{lueckmann2021benchmarking,
  title={Benchmarking Simulation-Based Inference},
  author={Lueckmann, Jan-Matthis and Boelts, Jan and Greenberg, David and Goncalves, Pedro and Macke, Jakob},
  booktitle={Proceedings of the 24th International Conference on Artificial Intelligence and Statistics},
  pages={343--351}, year={2021}, volume={130}, series={PMLR}
}

@article{cranmer2020frontier,
  title={The Frontier of Simulation-Based Inference},
  author={Cranmer, Kyle and Brehmer, Johann and Louppe, Gilles},
  journal={Proceedings of the National Academy of Sciences},
  volume={117}, number={48}, pages={30055--30062}, year={2020}
}

@inproceedings{greenberg2019automatic,
  title={Automatic Posterior Transformation for Likelihood-Free Inference},
  author={Greenberg, David and Nonnenmacher, Marcel and Macke, Jakob},
  booktitle={Proceedings of the 36th International Conference on Machine Learning},
  pages={2404--2414}, year={2019}, volume={97}, series={PMLR}
}

@article{radev2020bayesflow,
  title={BayesFlow: Learning Complex Stochastic Models with Invertible Neural Networks},
  author={Radev, Stefan T. and Mertens, Ulf K. and Voss, Andreas and Ardizzone, Lynton and K{\"o}the, Ullrich},
  journal={IEEE Transactions on Neural Networks and Learning Systems},
  year={2022}, volume={33}, number={4}, pages={1452--1466}
}

@inproceedings{zaheer2017deep,
  title={Deep Sets},
  author={Zaheer, Manzil and Kottur, Satwik and Ravanbakhsh, Siamak and Poczos, Barnabas and Salakhutdinov, Ruslan and Smola, Alexander},
  booktitle={Advances in Neural Information Processing Systems}, volume={30}, year={2017}
}

@article{talts2018validating,
  title={Validating Bayesian Inference Algorithms with Simulation-Based Calibration},
  author={Talts, Sean and Betancourt, Michael and Simpson, Daniel and Vehtari, Aki and Gelman, Andrew},
  journal={arXiv preprint arXiv:1804.06788}, year={2018}
}

@book{vandervaart1998asymptotic,
  title={Asymptotic Statistics}, author={van der Vaart, Aad W.},
  publisher={Cambridge University Press}, year={1998}
}

@inproceedings{papamakarios2016fast,
  title={Fast $\epsilon$-free Inference of Simulation Models with Bayesian Conditional Density Estimation},
  author={Papamakarios, George and Murray, Iain},
  booktitle={Advances in Neural Information Processing Systems}, volume={29}, year={2016}
}

@inproceedings{lueckmann2017flexible,
  title={Flexible Statistical Inference for Mechanistic Models of Neural Dynamics},
  author={Lueckmann, Jan-Matthis and Goncalves, Pedro J. and Bassetto, Giacomo and {\"O.}cal, Kaan and Nonnenmacher, Marcel and Macke, Jakob H.},
  booktitle={Advances in Neural Information Processing Systems}, volume={30}, year={2017}
}

@inproceedings{papamakarios2019sequential,
  title={Sequential Neural Likelihood: Fast Likelihood-free Inference with Autoregressive Flows},
  author={Papamakarios, George and Sterratt, David and Murray, Iain},
  booktitle={Proceedings of the 22nd International Conference on Artificial Intelligence and Statistics}, year={2019}
}

@inproceedings{hermans2020likelihoodfree,
  title={Likelihood-free MCMC with Amortized Approximate Ratio Estimators},
  author={Hermans, Joeri and Begy, Volodimir and Louppe, Gilles},
  booktitle={Proceedings of the 37th International Conference on Machine Learning}, year={2020}
}

@inproceedings{deistler2022truncated,
  title={Truncated Proposals for Scalable and Hassle-free Simulation-based Inference},
  author={Deistler, Michael and Goncalves, Pedro J. and Macke, Jakob H.},
  booktitle={Advances in Neural Information Processing Systems}, volume={35}, year={2022}
}

@inproceedings{durkan2019neural,
  title={Neural Spline Flows},
  author={Durkan, Conor and Bekasov, Artur and Murray, Iain and Papamakarios, George},
  booktitle={Advances in Neural Information Processing Systems}, volume={32}, year={2019}
}

@inproceedings{gloeckler2024allinone,
  title={All-in-one Simulation-based Inference},
  author={Gloeckler, Manuel and Deistler, Michael and Weilbach, Christian and Wood, Frank and Macke, Jakob H.},
  booktitle={Proceedings of the 41st International Conference on Machine Learning}, year={2024}
}

@inproceedings{lee2019set,
  title={Set Transformer: A Framework for Attention-based Permutation-Invariant Neural Networks},
  author={Lee, Juho and Lee, Yoonho and Kim, Jungtaek and Kosiorek, Adam and Choi, Seungjin and Teh, Yee Whye},
  booktitle={Proceedings of the 36th International Conference on Machine Learning}, year={2019}
}

@inproceedings{chan2018likelihood,
  title={A Likelihood-free Inference Framework for Population Genetic Data Using Exchangeable Neural Networks},
  author={Chan, Jeffrey and Perrone, Valerio and Spence, Jeffrey P. and Jenkins, Paul A. and Mathieson, Sara and Song, Yun S.},
  booktitle={Advances in Neural Information Processing Systems}, volume={31}, year={2018}
}

@article{beaumont2002approximate,
  title={Approximate Bayesian Computation in Population Genetics},
  author={Beaumont, Mark A. and Zhang, Wenyang and Balding, David J.},
  journal={Genetics}, volume={162}, number={4}, pages={2025--2035}, year={2002}
}

@article{fearnhead2012constructing,
  title={Constructing Summary Statistics for Approximate Bayesian Computation: Semi-automatic Approximate Bayesian Computation},
  author={Fearnhead, Paul and Prangle, Dennis},
  journal={Journal of the Royal Statistical Society: Series B}, volume={74}, number={3}, pages={419--474}, year={2012}
}

@article{hermans2022trust,
  title={A Crisis in Simulation-Based Inference? Beware, Your Posterior Approximations Can Be Unfaithful},
  author={Hermans, Joeri and Delaunoy, Arnaud and Rozet, Fran{\c{c}}ois and Wehenkel, Antoine and Begy, Volodimir and Louppe, Gilles},
  journal={Transactions on Machine Learning Research}, year={2022}
}

@inproceedings{ward2022robust,
  title={Robust Neural Posterior Estimation and Statistical Model Criticism},
  author={Ward, Daniel and Cannon, Patrick and Beaumont, Mark and Fasiolo, Matteo and Schmon, Sebastian M.},
  booktitle={Advances in Neural Information Processing Systems}, volume={35}, year={2022}
}

@inproceedings{lopez2017revisiting,
  title     = {Revisiting Classifier Two-Sample Tests},
  author    = {Lopez-Paz, David and Oquab, Maxime},
  booktitle = {International Conference on Learning Representations},
  year      = {2017}
}

@article{gretton2012kernel,
  title={A kernel two-sample test},
  author={Gretton, Arthur and Borgwardt, Karsten M. and Rasch, Malte J. and Sch{\"o}lkopf, Bernhard and Smola, Alexander},
  journal={Journal of Machine Learning Research},
  volume={13},
  pages={723--773},
  year={2012}
}

@article{frazier2024statistical,
  title={The Statistical Accuracy of Neural Posterior and Likelihood Estimation},
  author={Frazier, David T. and Kelly, Ryan and Drovandi, Christopher and Warne, David J.},
  journal={arXiv preprint arXiv:2411.12068},
  year={2024},
  url={https://arxiv.org/abs/2411.12068}
}

@article{maceda2026variational,
  title={Demonstrating the power and flexibility of variational assumptions for amortized neural posterior estimation in environmental applications},
  author={Maceda, Elliot and Hector, Emily C. and Lenzi, Amanda and Reich, Brian J.},
  journal={Environmetrics},
  volume={37}, number={1}, pages={e70074},
  year={2026},
  doi={10.1002/env.70074}
}

@article{linhart2026diffusion,
  title={Diffusion posterior sampling for simulation-based inference in tall data settings},
  author={Linhart, Julia and Cardoso, Gabriel Victorino and Gramfort, Alexandre and Le Corff, Sylvain and Rodrigues, Pedro L. C.},
  journal={Transactions on Machine Learning Research},
  year={2026},
  url={https://arxiv.org/abs/2404.07593}
}

@inproceedings{durkan2020contrastive,
  title={On contrastive learning for likelihood-free inference},
  author={Durkan, Conor and Murray, Iain and Papamakarios, George},
  booktitle={International Conference on Machine Learning},
  pages={2771--2781},
  year={2020},
  organization={PMLR}
}

@article{goncalves2020training,
  title={Training deep neural density estimators to identify mechanistic models of neural dynamics},
  author={Gon{\c{c}}alves, Pedro J. and Lueckmann, Jan-Matthis and Deistler, Michael and Nonnenmacher, Marcel and {\"O}cal, Kaan and Bassetto, Giacomo and Chintaluri, Chaitanya and Podlaski, William F. and Haddad, Sara A. and Vogels, Tim P. and Greenberg, David S. and Macke, Jakob H.},
  journal={eLife},
  volume={9},
  pages={e56261},
  year={2020}
}

@article{kermack1927contribution,
  title={A contribution to the mathematical theory of epidemics},
  author={Kermack, William Ogilvy and McKendrick, Anderson G.},
  journal={Proceedings of the Royal Society of London. Series A},
  volume={115},
  number={772},
  pages={700--721},
  year={1927}
}

@article{lotka1920analytical,
  title={Analytical note on certain rhythmic relations in organic systems},
  author={Lotka, Alfred J.},
  journal={Proceedings of the National Academy of Sciences},
  volume={6},
  number={7},
  pages={410--415},
  year={1920}
}

@article{sisson2007sequential,
  title={Sequential {M}onte {C}arlo without likelihoods},
  author={Sisson, Scott A. and Fan, Yanan and Tanaka, Mark M.},
  journal={Proceedings of the National Academy of Sciences},
  volume={104},
  number={6},
  pages={1760--1765},
  year={2007}
}

@article{beaumont2009adaptive,
  title={Adaptive approximate {B}ayesian computation},
  author={Beaumont, Mark A. and Cornuet, Jean-Marie and Marin, Jean-Michel and Robert, Christian P.},
  journal={Biometrika},
  volume={96},
  number={4},
  pages={983--990},
  year={2009}
}

@article{kingman1982coalescent,
  title={The coalescent},
  author={Kingman, John Frank Charles},
  journal={Stochastic Processes and their Applications},
  volume={13},
  number={3},
  pages={235--248},
  year={1982}
}

@article{watterson1975number,
  title={On the number of segregating sites in genetical models without recombination},
  author={Watterson, G. A.},
  journal={Theoretical Population Biology},
  volume={7},
  number={2},
  pages={256--276},
  year={1975}
}

@article{auton2015global,
  title={A global reference for human genetic variation},
  author={{The 1000 Genomes Project Consortium}},
  journal={Nature},
  volume={526},
  number={7571},
  pages={68--74},
  year={2015}
}

@article{Bochkina2014,
author = {Natalia A. Bochkina and Peter J. Green},
title = {{The Bernstein--von Mises theorem and nonregular models}},
volume = {42},
journal = {The Annals of Statistics},
number = {5},
publisher = {Institute of Mathematical Statistics},
pages = {1850 -- 1878},
year = {2014},
doi = {10.1214/14-AOS1239},
URL = {https://doi.org/10.1214/14-AOS1239}
}

@inproceedings{kingma2015adam,
  title={Adam: A Method for Stochastic Optimization},
  author={Kingma, Diederik P. and Ba, Jimmy},
  booktitle={International Conference on Learning Representations},
  year={2015}
}

@article{elfwing2017sigmoid,
  title={Sigmoid-Weighted Linear Units for Neural Network Function Approximation in Reinforcement Learning},
  author={Elfwing, Stefan and Uchibe, Eiji and Doya, Kenji},
  journal={arXiv preprint arXiv:1702.03118},
  year={2017},
  url={https://arxiv.org/abs/1702.03118}
}

@inproceedings{loshchilov2017sgdr,
  title={{SGDR}: Stochastic Gradient Descent with Warm Restarts},
  author={Loshchilov, Ilya and Hutter, Frank},
  booktitle={International Conference on Learning Representations},
  year={2017},
  url={https://arxiv.org/abs/1608.03983}
}

@article{szekely2013energy,
  title={Energy statistics: A class of statistics based on distances},
  author={Sz{\'e}kely, G{\'a}bor J. and Rizzo, Maria L.},
  journal={Journal of Statistical Planning and Inference},
  volume={143}, number={8}, pages={1249--1272},
  year={2013},
  doi={10.1016/j.jspi.2013.03.018}
}

@article{habermann2025amortized,
  title={Amortized Bayesian Multilevel Models},
  author={Habermann, Daniel and Schmitt, Marvin and K{\"u}hmichel, Lars and Bulling, Andreas and Radev, Stefan T. and B{\"u}rkner, Paul-Christian},
  journal={Bayesian Analysis},
  year={2025},
  note={Advance publication, doi:10.1214/25-BA1570}
}

@article{mittal2025amortized,
  title={Amortized In-Context Bayesian Posterior Estimation},
  author={Mittal, Sarthak and Bracher, Niels Leif and Lajoie, Guillaume and Jaini, Priyank and Brubaker, Marcus},
  journal={arXiv preprint arXiv:2502.06601},
  year={2025}
}

@inproceedings{sun2026amortized,
  title={Amortized Simulation-Based Inference in Generalized Bayes via Neural Posterior Estimation},
  author={Sun, Shiyi and Nicholls, Geoff K. and Lee, Jeong Eun},
  booktitle={Proceedings of the 43rd International Conference on Machine Learning},
  year={2026}
}
\bibliographystyle{iclr2027_conference}

\clearpage
\appendix
\raggedbottom
\begin{center}{\Large\bfseries Appendix}\end{center}
\vspace{0.5em}
\noindent\textbf{Roadmap.}  Appendix~\ref{app:details} gives the architecture and training protocol; Appendix~\ref{app:center-theory} the statements and proofs behind Section~\ref{sec:anatomy}; Appendix~\ref{app:benchmarks} the benchmark and reference definitions; Appendix~\ref{app:regular} the six regular benchmarks at every size; Appendix~\ref{app:boundary} the nonregular stress tests; Appendix~\ref{app:diagnostics} the controls (best-effort baselines, a second training maximum, fixed-center, exponent estimation, size context, a compositional comparator); Appendix~\ref{app:popgen} the real-data application.
\vspace{0.5em}
\section{Implementation and training protocol}
\label{app:details}

The five estimators share the encoder, the optimization protocol and
the training budget; they differ only in the density head and in how
the center is obtained (Section~\ref{sec:methods}).

\subsection{Architecture}
All five estimators use the same Deep Sets encoder \citep{zaheer2017deep}:
\begin{itemize}
\item \textbf{Per-observation MLP:} $x_i\in\mathbb R^{d_x}\mapsto h_i\in\mathbb R^{64}$,
with widths $d_x\to64\to64\to64$.
\item \textbf{Mean pooling:} $\bar h=n^{-1}\sum_{i=1}^n h_i$.
\item \textbf{Aggregation MLP:} $[\bar h,t_n]\mapsto s_n\in\mathbb R^{64}$,
with widths $(64+d_t)\to64\to64$.
\end{itemize}
Here $t_n\in\mathbb R^{d_t}$ contains the sample mean, log SD, skewness
and excess kurtosis of each coordinate, plus task-specific summaries
for SIR and Lotka--Volterra (Appendix~\ref{app:benchmarks}).

The density heads and center networks are:
\begin{itemize}
\item \textbf{Gaussian head:} MLP with hidden widths $(64,64)$;
outputs $\mu\in\mathbb R^d$ and a lower-triangular $L$, with
$\Sigma=LL^\top$ and $L_{jj}=\exp(\operatorname{clip}(a_j,-7,5))$.
\item \textbf{NSF head} \citep{durkan2019neural}: 5 transforms, 12 bins, conditioner hidden widths $(64,64)$.
\item \textbf{Center/correction:} MLP with hidden widths $(64,64)$ and $d$ outputs.
\end{itemize}
All MLPs use SiLU \citep{elfwing2017sigmoid} between linear layers and no output activation.

\textbf{Parameter coordinates.} Free-G uses $\theta$; Free-NSF uses
$(\theta-\mu_{\mathrm{prior}})/\sigma_{\mathrm{prior}}$, with coordinate-wise
moments estimated from 20,000 prior draws. Scale-G, CE-G and CE-NSF use
$z_n=n^\alpha(\theta-c_n)$ (Equation~\ref{eq:residual}) without further standardization.

\subsection{Size inputs}
\label{app:size-context}
We use $\log n$ to express neural size conditioning in the main-text
parameterization. Sample size also enters through the explicit
contraction factor $n^{-\alpha}$ and the statistical anchor computed
from the observed dataset. The regular benchmarks use $\alpha=1/2$.

Appendix~\ref{app:residual-context} compares $\log n$, $1/n$,
and no size input to the residual network while holding the structured center and
contraction factor fixed. On this Gamma diagnostic, the three choices
give nearly identical contraction exponents and similar width and
calibration diagnostics, with small differences in KL. This comparison
concerns residual context; it does not establish equivalence of all
size encodings or network components.

\subsection{Training and evaluation}
\begin{itemize}
\item Training: up to 4000 steps, each on a fresh batch of 128 simulated
datasets, $n$ drawn uniformly from $\{5,\dots,20\}$; Adam
\citep{kingma2015adam} with learning rate $10^{-3}$, cosine-annealed to
zero over the run without restarts \citep{loshchilov2017sgdr}, and gradient norm clipped at $5$. The objective is
the negative log density of the simulated parameter under the estimator
(Section~\ref{sec:methods}).
\item Validation: 192 fixed cases at each $n\in\{5,10,15,20\}$ every
200 steps. Validation NLL is averaged across these sizes; the best
checkpoint is retained. An improvement must exceed $10^{-4}$, and
training stops after 20 validation checks without such an improvement.
\item Evaluation: 400 paired cases at $n\in\{10,20,50,100\}$, identical
across methods and seeds; 256 posterior draws for calibration
metrics, 1024 for reference metrics (Section~\ref{sec:protocol}).
\item Out-of-range sizes are never used for early stopping,
hyperparameter selection or post-hoc calibration.
\item Seeds: three on the regular benchmarks, except the mixture,
which uses five for C2ST and coverage and two for the saved width and
center-error summaries; the prototypes use ten. Values are means across
seeds, with sample SDs where shown; table captions state exceptions.
\end{itemize}
The prototypes of Appendix~\ref{app:boundary} use the same training
range and evaluation grid with a smaller network and their own budget,
stated there. Two protocol variations are reported separately:
best-effort free baselines (learning-rate sweep, three times the
budget; Appendix~\ref{app:free-best}) and a second training maximum
$N=80$ on SIR (training $n\in\{20,\dots,80\}$, evaluation
$n\in\{40,80,200,400\}$; Appendix~\ref{app:sir-n80}).

\clearpage
\section{Theoretical support}
\label{app:center-theory}

This appendix collects the statements and proofs behind Section~\ref{sec:anatomy}.
The first part concerns the center: what a contracting anchor means, why the
learned correction is multiplied by the contraction factor, and how center and
rate errors change with sample size. The second part concerns the fitted
posterior as a whole: an exact decomposition of its error into shape, mean and
covariance terms. Throughout, the network is fixed after training.

\subsection{Notation and the meaning of a contracting center}
Let $P_n$ be the true posterior for a dataset of size $n$, with mean
$m_n$ and covariance $n^{-2\alpha}\Sigma_n$, where
$\Sigma_n\to\Sigma\succ0$. The exponent $\alpha>0$ specifies the
contraction rate; $\Sigma_n$ specifies the remaining covariance.
All probability statements concern repeated datasets at a fixed
generating parameter $\theta_0$ after the network has been trained.

We distinguish three locations. The statistical anchor is
$c_{\mathrm{stat}}(x_{1:n})$. The residual coordinate is centered at
$c_n=c_{\mathrm{stat}}+n^{-\alpha}\Delta_n$, where $\Delta_n$ is the
network's local correction. The mean of the final fitted posterior is
$a_n$. If its residual has mean $\mu_n$, then
\begin{equation}
a_n=c_{\mathrm{stat}}+n^{-\alpha}(\Delta_n+\mu_n).
\label{eq:full-ce-mean}
\end{equation}
Thus, the anchor and the fitted posterior mean need not coincide.
Only the sum of the two learned location terms affects $a_n$; an
implementation may combine them into a single output.

The condition $n^\alpha(c_{\mathrm{stat}}-m_n)=O_p(1)$ means that the
anchor's error remains bounded in probability after multiplying by
$n^\alpha$. Because $\Sigma_n$ has a positive-definite limit, this is
equivalent to bounded error in posterior-standard-deviation units.
It does not require that the standardized error tend to zero.

\subsection{Why multiply the correction by the contraction factor?}
\begin{proposition}[Scaling of center errors]
\label{prop:center-coordinates}
Suppose $n^\alpha(c_{\mathrm{stat}}-m_n)=O_p(1)$ and the fitted local
correction satisfies $\Delta_n=O_p(1)$. Then
\[
c_n=c_{\mathrm{stat}}+n^{-\alpha}\Delta_n
\quad\Longrightarrow\quad
n^\alpha(c_n-m_n)=O_p(1).
\]
If instead the correction is applied in parameter units, then
\[
c_n=c_{\mathrm{stat}}+\Delta_n
\quad\Longrightarrow\quad
n^\alpha(c_n-m_n)=n^\alpha\Delta_n+O_p(1).
\]
In particular, a correction converging to a nonzero constant makes the
standardized center error diverge. A parameter-unit correction of order
$O_p(n^{-\alpha})$ retains bounded error.
\end{proposition}
\begin{proof}
Subtract $m_n$ and multiply each expression by $n^\alpha$.
For the local correction the result is
$n^\alpha(c_{\mathrm{stat}}-m_n)+\Delta_n$, a sum of two tight terms.
For the parameter-unit correction the second term is
$n^\alpha\Delta_n$, giving the stated alternatives.
\end{proof}

This is a sufficient-condition statement about the correction, not
proof that a particular training procedure meets it. The result also
requires a contracting anchor. If the anchor retains a fixed bias,
a bounded local correction becomes too small in parameter units to
cancel it. For the final mean $a_n$, the same argument applies to
$\Delta_n+\mu_n$ in Equation~\ref{eq:full-ce-mean}.

\subsection{How far can center and rate errors grow?}
Set $N=\nmax$ and write a deployment size as $n=fN$, with $f\ge1$.
The following calculation assumes specified power laws over the range
of interest. The width $\sigma_n$ is a scalar reference scale; it may
describe one parameter direction. It is not a substitute for checking
the full covariance in a multidimensional problem.

\begin{proposition}[Center and rate sensitivity]
\label{prop:horizon}
Suppose $\|a_{fN}-m_{fN}\|\le e_N f^{-\gamma}$ and
$\sigma_{fN}=\sigma_N f^{-\alpha}$. Then
\begin{equation}
\frac{\|a_{fN}-m_{fN}\|}{\sigma_{fN}}
\le\frac{e_N}{\sigma_N}f^{\alpha-\gamma}.
\label{eq:center-sensitivity}
\end{equation}
If the center-error law is an equality with $e_N>0$ and
$\alpha>\gamma$, a tolerance $\tau\ge e_N/\sigma_N$ is reached at
\[
f_{\mathrm{center}}=(\tau\sigma_N/e_N)^{1/(\alpha-\gamma)}.
\]
Separately, suppose the fitted width is
$\widehat\sigma_{fN}=\sigma_N f^{-\widehat\alpha}$, so that it matches
the reference at $N$. With $\eta=\widehat\alpha-\alpha$,
\begin{equation}
\frac{\widehat\sigma_{fN}}{\sigma_{fN}}=f^{-\eta},
\qquad
\left|\log\frac{\widehat\sigma_{fN}}{\sigma_{fN}}\right|\le\kappa
\ \Longleftrightarrow\ |\eta|\log f\le\kappa.
\label{eq:rate-sensitivity}
\end{equation}
\end{proposition}
\begin{proof}
Divide the center-error bound by the reference width. Under equality,
solve for $f$ at the specified tolerance. For the width statement,
divide the two power laws and take logarithms.
\end{proof}

Equation~\ref{eq:center-sensitivity} has three cases: the bound on
standardized error decreases if $\gamma>\alpha$, stays constant if
$\gamma=\alpha$, and grows if $\gamma<\alpha$. These describe the
error itself when the assumed center-error law is an equality; a
growing upper bound alone does not establish that the error grows.
With a fixed offset and root-$n$ contraction, the error grows as
$\sqrt f$ under equality. For example, increasing sample size from
$N$ to $25N$ multiplies the standardized error by five, even though
the absolute offset has not changed. The threshold
$f_{\mathrm{center}}$ is therefore conditional on the assumed power
laws, not an extrapolation range guaranteed by training. Here the
error concerns the final fitted mean $a_n$, including any residual
mean, rather than the coordinate anchor alone.
Equation~\ref{eq:rate-sensitivity} says that an exponent that is too
large produces a posterior that becomes too narrow, and conversely
for an exponent that is too small. The T1 prototype (Appendix~\ref{app:boundary}) is a direct test:
its posterior contracts at rate $n^{-1}$ while every estimator imposes
$n^{-1/2}$, so $\eta=-1/2$ and the predicted over-dispersion is
$\sqrt f$. The measured factors on T1 are $1.56$ and $2.21$ at $f=2.5$
and $f=5$, within $2\%$ of the prediction for both residual families.

The fitted-width assumption is essential. A size-dependent residual
can change its own width and partly undo an imposed exponent; the
coordinate transformation alone does not imply that the fitted width
follows the stated power law. These calculations also omit other
sources of posterior error. They explain sensitivity under specified
conditions, rather than establishing a deployment guarantee.

\subsection{An exact decomposition for a Gaussian approximation}
\label{app:transfer}
The true posterior
$P_n$ has mean $m_n$ and covariance $n^{-2\alpha}\Sigma_n$. The fitted
Gaussian $Q_n$ has mean $a_n$ and covariance $n^{-2\alpha}A_n$.
The rate factor is shared in this notation, but $A_n$ is allowed to
depend on $n$, so the notation does not assume correct fitted contraction.

\begin{proposition}[Exact error decomposition]
\label{prop:finite-error}
Suppose $P_n$ has a density with finite differential entropy and
positive-definite covariance, and $A_n\succ0$. Define
\[
G_n=\mathcal N(m_n,n^{-2\alpha}\Sigma_n),\qquad
J_n=\mathrm{KL}(P_n\Vert G_n),
\]
and set $h_n=n^\alpha(a_n-m_n)$ and
$E_n=A_n^{-1/2}\Sigma_n A_n^{-1/2}$. Then
\begin{equation}
\begin{split}
\mathrm{KL}(P_n\Vert Q_n)
={}&\underbrace{J_n}_{\text{Gaussian shape error}}
+\underbrace{\tfrac12 h_n^\top A_n^{-1}h_n}_{\text{mean error}}\\
&+\underbrace{\tfrac12\{\operatorname{tr}(E_n)-d-\log\det E_n\}}_
{\text{covariance error}}.
\end{split}
\label{eq:finite-error}
\end{equation}
All three terms are nonnegative. If the eigenvalues of $A_n$ and $E_n$
stay in fixed compact subsets of $(0,\infty)$, there are constants
$0<c<C<\infty$ such that
\[
c\{\|h_n\|^2+\|E_n-I\|_F^2\}
\le\mathrm{KL}(P_n\Vert Q_n)-J_n
\le C\{\|h_n\|^2+\|E_n-I\|_F^2\}.
\]
\end{proposition}
\begin{proof}
Add and subtract $\log g_n$ inside $\E_{P_n}\log(p_n/q_n)$.
The first term is $J_n$. The remaining expectation involves the log
ratio of two Gaussian densities, which is quadratic in $\theta$.
Since $P_n$ and $G_n$ have the same first two moments, that expectation
equals $\mathrm{KL}(G_n\Vert Q_n)$. The Gaussian KL formula gives
Equation~\ref{eq:finite-error}.

For the quadratic bounds, diagonalize $E_n$. Its covariance penalty
is the sum of $g(\lambda)=\lambda-1-\log\lambda$ over its eigenvalues.
Here $g(1)=g'(1)=0$ and $g''(\lambda)=\lambda^{-2}$. Taylor's theorem
on the stated compact interval bounds this sum above and below by
constant multiples of $\|E_n-I\|_F^2$. The eigenvalue bounds for $A_n$
give the corresponding bounds for the mean term.
\end{proof}

The first term is the error that remains even after choosing the best
Gaussian mean and covariance. The second measures displacement in units
of the fitted uncertainty. The third includes errors in widths,
correlations and covariance orientation. None of these is removed merely
by writing an explicit factor $n^{-2\alpha}$.

For CE-G, $a_n$ includes both learned location terms in
Equation~\ref{eq:full-ce-mean}; substituting the statistical anchor alone
would omit part of the error. If $P_n$ is itself Gaussian, $J_n=0$.
For a flow approximation, KL remains invariant under an invertible
coordinate change, but the Gaussian three-term formula does not apply
to its general density.

\paragraph{Bounded error can still be substantial.}
For $P_n=\mathcal N(0,1/n)$ and
$Q_n=\mathcal N(a/\sqrt n,1/n)$, the formula gives
$\mathrm{KL}(P_n\Vert Q_n)=a^2/2$ for every $n$.
Both distributions contract at the same rate, yet their relative
displacement stays fixed. Bounded standardized mean error therefore
does not imply vanishing posterior error or nominal coverage.

\paragraph{The moment-matched Gaussian is the best Gaussian.}
Taking $a_n=m_n$ and $A_n=\Sigma_n$ in Equation~\ref{eq:finite-error}
makes the mean and covariance terms vanish, and both are strictly
positive otherwise, so $G_n$ is the unique minimizer of
$\mathrm{KL}(P_n\Vert\,\cdot\,)$ over Gaussians and $J_n$ is the
Gaussian family's irreducible error. Because an unrestricted Gaussian
remains an unrestricted Gaussian after an invertible centering and
scaling, Free-G, Scale-G and CE-G share this best achievable
approximation at a fixed dataset; their fitted errors differ because
they learn and extrapolate the mean and covariance differently. Only a
more flexible residual can reduce $J_n$, and whether a trained spline
does so is an empirical question (Section~\ref{sec:gmm-results}).

\clearpage
\section{Benchmark and reference definitions}
\label{app:benchmarks}
\setlength{\LTleft}{0pt}\setlength{\LTright}{\fill}%
\setlength{\LTpre}{4pt}\setlength{\LTpost}{4pt}%

\subsection{Gaussian linear}

\begin{longtable}{@{}lp{0.7\linewidth}@{}}
\textbf{Prior} & $\theta\sim\cN(0,\,0.1\,I_{10})$ \\[6pt]
\textbf{Simulator} & $x_i\mid\theta\sim\cN(\theta,\,0.1\,I_{10})$ \\[6pt]
\textbf{Dimensions} & $\theta\in\R^{10}$, $x_i\in\R^{10}$ \\[6pt]
\textbf{Center} & sample mean $\bar x$ \\[6pt]
\textbf{Reference posterior} & conjugate Gaussian posterior \\[6pt]
\textbf{References} & \citet{lueckmann2021benchmarking} \\
\end{longtable}

\subsection{Gaussian linear (uniform)}

\begin{longtable}{@{}lp{0.7\linewidth}@{}}
\textbf{Prior} & $\theta\sim U[-1,1]^{10}$ \\[6pt]
\textbf{Simulator} & $x_i\mid\theta\sim\cN(\theta,\,0.1\,I_{10})$ \\[6pt]
\textbf{Dimensions} & $\theta\in\R^{10}$, $x_i\in\R^{10}$ \\[6pt]
\textbf{Center} & sample mean $\bar x$ \\[6pt]
\textbf{Reference posterior} & product of truncated normals \\[6pt]
\textbf{References} & \citet{lueckmann2021benchmarking} \\
\end{longtable}

\subsection{Gaussian mixture}

\begin{longtable}{@{}lp{0.7\linewidth}@{}}
\textbf{Prior} & $\theta\sim U[-10,10]^2$ \\[6pt]
\textbf{Simulator} & $x_i\mid\theta\sim\tfrac12\cN(\theta,I)+\tfrac12\cN(\theta,0.01\,I)$ \\[6pt]
\textbf{Dimensions} & $\theta\in\R^{2}$, $x_i\in\R^{2}$ \\[6pt]
\textbf{Center} & ten EM steps from the coordinate-wise median \\[6pt]
\textbf{Reference posterior} & grid integration, $201\times201$ points around the MAP \\[6pt]
\textbf{References} & \citet{sisson2007sequential,beaumont2009adaptive} \\
\end{longtable}

\subsection{Bernoulli GLM}
Inference of a ten-parameter generalized linear model with Bernoulli
observations and a Gaussian prior whose covariance encourages
smoothness by penalizing second-order differences in the filter; the
observation is the sufficient statistic of the GLM.

\begin{longtable}{@{}lp{0.7\linewidth}@{}}
\textbf{Prior} & $\theta=(\beta,f)$: $\beta\sim\cN(0,2)$,
$f\sim\cN\bigl(0,(F^\top F)^{-1}\bigr)$, with $F_{i,i-2}=1$,
$F_{i,i-1}=-2$, $F_{i,i}=1+\sqrt{(i-1)/9}$, $F_{i,j}=0$ otherwise,
$1\le i,j\le9$ \\[6pt]
\textbf{Simulator} &
\begin{tabular}[t]{@{}l@{}}
$x_i=(x_{i,1},\dots,x_{i,10})$, $x_{i,1}=\sum_{t=1}^{T}z_t$, \\[5pt]
$x_{i,2:10}=\frac{1}{x_{i,1}}Vz$,
$z_t\sim\mathrm{Bern}\bigl(\eta(v_t^\top f+\beta)\bigr)$, \\[5pt]
$\eta(\cdot)=\exp(\cdot)/(1+\exp(\cdot))$, \\[5pt]
frozen input between time bins $t-8$ and $t$: \\[5pt]
$v_t\sim\cN(0,I)$, $V=[v_1,\dots,v_T]$
\end{tabular} \\[6pt]
\textbf{Dimensions} & $\theta\in\R^{10}$, $x_i\in\R^{10}$ \\[6pt]
\textbf{Fixed parameters} & duration $T=100$ time bins \\[6pt]
\textbf{Center} & simulator method of moments: damped Gauss--Newton
steps matching $X^\top\sigma(X\theta)$ to the replicate mean \\[6pt]
\textbf{Reference posterior} & adaptive importance sampling from the
Newton MAP  \\[6pt]
\textbf{References} & \citet{lueckmann2017flexible,goncalves2020training} \\
\end{longtable}

\subsection{SLCP}
A task designed to have a simple likelihood and a complex posterior:
one replicate is four two-dimensional points from a Gaussian whose
mean and covariance are nonlinear functions of $\theta$, so the
posterior is a union of four mirror images.

\begin{longtable}{@{}lp{0.7\linewidth}@{}}
\textbf{Prior} & $\theta\sim U[-3,3]^5$ \\[6pt]
\textbf{Simulator} &
\begin{tabular}[t]{@{}l@{}}
$x_i=(x_{i,1},\dots,x_{i,4})$,
$x_{i,j}\sim\cN(m_\theta,S_\theta)$, where \\[5pt]
$m_\theta=(\theta_1,\theta_2)^\top$,
$S_\theta=\bigl(\begin{smallmatrix}s_1^2 & \rho s_1s_2\\
\rho s_1s_2 & s_2^2\end{smallmatrix}\bigr)$, \\[5pt]
$s_1=\theta_3^2$, $s_2=\theta_4^2$, $\rho=\tanh\theta_5$
\end{tabular} \\[6pt]
\textbf{Dimensions} & $\theta\in\R^{5}$, $x_i\in\R^{8}$ \\[6pt]
\textbf{Center} & moment estimator on the branch $\theta_3,\theta_4>0$ \\[6pt]
\textbf{Reference posterior} & importance sampling on a fundamental
domain \\[6pt]
\textbf{References} &
\citet{papamakarios2019sequential,greenberg2019automatic,hermans2020likelihoodfree,durkan2020contrastive} \\
\end{longtable}

\subsection{Two moons}

\begin{longtable}{@{}lp{0.7\linewidth}@{}}
\textbf{Prior} & $\theta\sim U[-1,1]^2$ \\[6pt]
\textbf{Simulator} & $x_i\mid\theta=
\begin{bmatrix}r\cos\alpha+0.25\\ r\sin\alpha\end{bmatrix}+
\begin{bmatrix}-|\theta_1+\theta_2|/\sqrt2\\
(-\theta_1+\theta_2)/\sqrt2\end{bmatrix}$, where
$\alpha\sim U(-\pi/2,\pi/2)$, $r\sim\cN(0.1,\,0.01^2)$ \\[6pt]
\textbf{Dimensions} & $\theta\in\R^{2}$, $x_i\in\R^{2}$ \\[6pt]
\textbf{Center} & moment estimator on the branch $\theta_1+\theta_2\ge0$ \\[6pt]
\textbf{Reference posterior} & grid integration on a fundamental
domain, expanded by the symmetry group \\[6pt]
\textbf{References} & \citet{greenberg2019automatic} \\
\end{longtable}

\subsection{SIR}
An epidemiological model of susceptible $S$, infectious $I$ and
recovered $R$ individuals; the task infers the contact rate $\beta$
and the recovery rate $\gamma$ from binomial samples of $I$ at ten
evenly spaced times.

\begin{longtable}{@{}lp{0.7\linewidth}@{}}
\textbf{Prior} &
\begin{tabular}[t]{@{}l@{}}
$\beta\sim\mathrm{LogNormal}(\log 0.4,\,0.5)$,
$\gamma\sim\mathrm{LogNormal}(\log 0.125,\,0.2)$, \\[5pt]
truncated to $R_0=\beta/\gamma\in[1.5,8]$ (about $12\%$ of prior mass \\[2pt]
removed; below $1.5$ no epidemic occurs within the horizon, above $8$ \\[2pt]
the peak precedes the first observation and $\beta$ is weakly identified). \\[2pt]
Inference is in $(\log\beta,\log\gamma)$.
\end{tabular} \\[6pt]
\textbf{Simulator} &
\begin{tabular}[t]{@{}l@{}}
$x_i=(x_{i,1},\dots,x_{i,10})$,
$x_{i,t}\sim\mathrm{Binomial}\bigl(1000,\,\frac{I(t)}{N}\bigr)$ \\[5pt]
stored as fractions, where $I$ is simulated from \\[5pt]
$\frac{dS}{dt}=-\beta\frac{SI}{N}$ \\[5pt]
$\frac{dI}{dt}=\beta\frac{SI}{N}-\gamma I$ \\[5pt]
$\frac{dR}{dt}=\gamma I$
\end{tabular} \\[6pt]
\textbf{Dimensions} & $\theta\in\R^{2}$, $x_i\in\R^{10}$ \\[6pt]
\textbf{Fixed parameters} & population $N=10^6$, duration $T=160$, initial conditions
$(S(0),I(0),R(0))=(N-1,1,0)$ \\[6pt]
\textbf{Center} & simulator method of moments: coarse grid, then three
damped Gauss--Newton steps matching the simulator's mean output to the
replicate mean \\[6pt]
\textbf{Reference posterior} & grid integration in whitened Laplace
coordinates; Gaussian to high accuracy, fitted contraction exponent
$0.51$ \\[6pt]
\textbf{References} & \citet{kermack1927contribution} \\
\end{longtable}

\subsection{Lotka--Volterra}
The classic model of two interacting species; the task infers four
interaction parameters from both populations observed at ten evenly
spaced times.

\begin{longtable}{@{}lp{0.7\linewidth}@{}}
\textbf{Prior} &
\begin{tabular}[t]{@{}l@{}}
$\alpha\sim\mathrm{LogNormal}(-0.125,\,0.5)$,
$\beta\sim\mathrm{LogNormal}(-3,\,0.5)$, \\[5pt]
$\gamma\sim\mathrm{LogNormal}(-0.125,\,0.5)$,
$\delta\sim\mathrm{LogNormal}(-3,\,0.5)$, \\[5pt]
truncated to parameters for which neither species falls below \\[2pt]
$e^{-5}$ individuals at any observation time (about $8\%$ of prior \\[2pt]
mass removed; an effectively extinct species carries no \\[2pt]
information about its rates). Inference is in log-rate coordinates.
\end{tabular} \\[6pt]
\textbf{Simulator} &
\begin{tabular}[t]{@{}l@{}}
$x_{i,1,t}\sim\mathrm{LogNormal}(\log X(t),\,0.1)$, \\[5pt]
$x_{i,2,t}\sim\mathrm{LogNormal}(\log Y(t),\,0.1)$ \\[5pt]
stored as log-counts, where $X,Y$ are simulated from \\[5pt]
$\frac{dX}{dt}=\alpha X-\beta XY$ \\[5pt]
$\frac{dY}{dt}=-\gamma Y+\delta XY$
\end{tabular} \\[6pt]
\textbf{Dimensions} & $\theta\in\R^{4}$, $x_i\in\R^{20}$ \\[6pt]
\textbf{Fixed parameters} & duration $T=20$ time units, initial conditions $(X(0),Y(0))=(30,1)$ \\[6pt]
\textbf{Center} & simulator method of moments: $13^4$ grid, then four
damped Gauss--Newton steps with multi-start \\[6pt]
\textbf{Reference posterior} & adaptive importance sampling with a
Student-$t$ proposal at the Gauss--Newton MAP; Gaussian to high
accuracy, fitted contraction exponent $0.50$ \\[6pt]
\textbf{References} & \citet{lotka1920analytical} \\
\end{longtable}

\paragraph{Simulator-based centers (SIR and Lotka--Volterra).}
Both centers are likelihood-free method-of-moments estimates computed per dataset from simulator runs only.  \emph{Objective:} the simulator's mean output (the ODE solution with observation noise off) is fitted to the replicate mean by weighted least squares, with weights $n$ over the replicates' sample variance per output coordinate (floored at $10^{-3}$ of its mean), plus the Gaussian prior on $\theta$ as a quadratic penalty; the likelihood is never evaluated.  \emph{Initialization:} a coarse grid over the prior box, one batched ODE solve for the whole grid ($41\times41$ points within $\pm3.5$ prior SDs for SIR; $13^4$ points for Lotka--Volterra), scored by the same objective.  \emph{Refinement:} damped Gauss--Newton with a central-difference Jacobian and step halving, three steps from the best grid point for SIR and four steps from each of the eight best grid points for Lotka--Volterra, keeping the best fit (the oscillatory mean map has local optima at the wrong phase).  The same routine is applied at every $n$; no out-of-range data, reference posterior or tuning enters the center.

The Gamma model is defined with the ablation that uses it (Appendix~\ref{app:diagnostics}).

\clearpage
\section{Full regular-benchmark results}
\label{app:regular}

This appendix reports the six regular \texttt{sbibm} benchmarks at every evaluation size, in-range sizes next to the out-of-range sizes, under the unified protocol of Appendix~\ref{app:details}.  The two symmetric multimodal benchmarks are in Appendix~\ref{app:boundary}.

\begin{table}[!htb]
\centering
\caption{\textbf{Results on the six regular benchmarks at every evaluation size} (three seeds, except the mixture: five for C2ST and coverage, two for SD ratio; training $n\in\{5,\dots,20\}$). The two symmetric multimodal benchmarks are in Table~\ref{tab:sbibm-symmetric}. Bold marks the best entry in each column and block (closest to the target where the metric has one).}
\label{tab:sbibm-full}
\setlength{\tabcolsep}{4pt}
\renewcommand{\arraystretch}{1.25}
\resizebox{\linewidth}{!}{\begin{tabular}{l rrrr rrrr rrrr}
\toprule
 & \multicolumn{4}{c}{C2ST} & \multicolumn{4}{c}{SD ratio} & \multicolumn{4}{c}{$95\%$ joint coverage} \\
\cmidrule(lr){2-5}\cmidrule(lr){6-9}\cmidrule(lr){10-13}
estimator & $n=10$ & $n=20$ & $n=50$ & $n=100$ & $n=10$ & $n=20$ & $n=50$ & $n=100$ & $n=10$ & $n=20$ & $n=50$ & $n=100$ \\
\midrule
\multicolumn{13}{l}{\textit{Gaussian linear}} \\
Free-G & 0.57 & 0.61 & 0.75 & 0.89 & 1.02 & 1.09 & 1.24 & 1.41 & \textbf{0.96} & 0.97 & 0.97 & 0.86 \\
Scale-G & 0.55 & 0.57 & 0.65 & 0.71 & 1.01 & \textbf{1.00} & \textbf{1.00} & \textbf{1.00} & \textbf{0.96} & \textbf{0.95} & 0.91 & 0.87 \\
CE-G & \textbf{0.51} & \textbf{0.51} & 0.55 & 0.57 & \textbf{1.00} & \textbf{1.00} & 1.03 & 1.06 & \textbf{0.96} & 0.96 & \textbf{0.96} & 0.97 \\
Free-NSF & 0.55 & 0.58 & 0.72 & 0.87 & 1.01 & 1.06 & 1.26 & 1.53 & \textbf{0.96} & 0.98 & 0.99 & 1.00 \\
CE-NSF & \textbf{0.51} & \textbf{0.51} & \textbf{0.54} & \textbf{0.56} & 1.02 & 1.01 & 0.99 & 0.98 & 0.97 & 0.96 & 0.93 & \textbf{0.93} \\
\addlinespace[4pt]
\multicolumn{13}{l}{\textit{Gaussian linear (uniform)}} \\
Free-G & 0.84 & 0.88 & 0.94 & 0.98 & 1.62 & 1.97 & 2.72 & 3.62 & \textbf{0.96} & 0.98 & 0.90 & 0.55 \\
Scale-G & 0.66 & 0.67 & 0.68 & 0.74 & 1.06 & 1.03 & 1.02 & \textbf{1.02} & \textbf{0.96} & \textbf{0.94} & 0.92 & 0.90 \\
CE-G & 0.61 & 0.61 & 0.58 & 0.57 & \textbf{1.03} & \textbf{1.01} & 1.02 & 1.04 & \textbf{0.96} & \textbf{0.94} & \textbf{0.95} & 0.97 \\
Free-NSF & \textbf{0.58} & 0.62 & 0.78 & 0.92 & \textbf{1.03} & 1.12 & 1.35 & 1.62 & \textbf{0.96} & 0.98 & 0.99 & 0.99 \\
CE-NSF & 0.59 & \textbf{0.59} & \textbf{0.56} & \textbf{0.56} & 1.04 & 1.02 & \textbf{1.00} & \textbf{0.98} & \textbf{0.96} & \textbf{0.94} & 0.94 & \textbf{0.94} \\
\addlinespace[4pt]
\multicolumn{13}{l}{\textit{Gaussian mixture}} \\
Free-G & 0.87 & 0.90 & 0.93 & 0.95 & 2.73 & 3.34 & 4.54 & 5.86 & 0.93 & 0.93 & 0.97 & 0.98 \\
Scale-G & 0.87 & 0.89 & 0.90 & 0.92 & 2.81 & 2.83 & 2.74 & 2.70 & \textbf{0.94} & 0.89 & 0.81 & 0.75 \\
CE-G & 0.64 & 0.55 & 0.53 & 0.54 & \textbf{1.22} & \textbf{1.03} & 0.90 & 0.83 & \textbf{0.94} & \textbf{0.94} & 0.92 & 0.88 \\
Free-NSF & 0.87 & 0.89 & 0.93 & 0.95 & 3.14 & 3.44 & 4.08 & 4.76 & \textbf{0.94} & \textbf{0.94} & 0.93 & 0.86 \\
CE-NSF & \textbf{0.56} & \textbf{0.52} & \textbf{0.51} & \textbf{0.51} & 1.33 & 1.15 & \textbf{1.05} & \textbf{1.01} & 0.92 & \textbf{0.94} & \textbf{0.94} & \textbf{0.93} \\
\addlinespace[4pt]
\multicolumn{13}{l}{\textit{Bernoulli GLM}} \\
Free-G & 1.00 & 1.00 & 1.00 & 1.00 & 6.56 & 8.88 & 13.98 & 19.49 & 0.93 & 0.94 & 0.93 & 0.88 \\
Scale-G & 0.99 & 1.00 & 1.00 & 1.00 & 3.80 & 4.04 & 4.10 & 4.07 & 0.94 & 0.82 & 0.24 & 0.06 \\
CE-G & 0.70 & 0.67 & 0.66 & 0.68 & \textbf{1.05} & \textbf{1.04} & 1.07 & 1.10 & 0.92 & 0.94 & \textbf{0.95} & 0.97 \\
Free-NSF & 0.98 & 0.99 & 1.00 & 1.00 & 2.24 & 2.97 & 4.42 & 6.02 & 0.94 & \textbf{0.95} & 0.96 & 0.91 \\
CE-NSF & \textbf{0.67} & \textbf{0.64} & \textbf{0.63} & \textbf{0.65} & \textbf{1.05} & \textbf{1.04} & \textbf{1.05} & \textbf{1.08} & \textbf{0.95} & \textbf{0.95} & \textbf{0.95} & \textbf{0.95} \\
\addlinespace[4pt]
\multicolumn{13}{l}{\textit{SIR}} \\
Free-G & 0.88 & 0.92 & 0.97 & 0.99 & 2.27 & 3.07 & 4.37 & 7.43 & 0.94 & \textbf{0.96} & 0.61 & 0.36 \\
Scale-G & 0.88 & 0.90 & 0.96 & 0.99 & 2.30 & 2.60 & 3.09 & 4.46 & \textbf{0.95} & 0.92 & 0.46 & 0.15 \\
CE-G & 0.64 & 0.58 & 0.66 & 0.77 & \textbf{1.13} & \textbf{1.08} & \textbf{1.27} & \textbf{2.07} & 0.94 & 0.93 & \textbf{0.95} & \textbf{0.97} \\
Free-NSF & 0.74 & 0.79 & 0.92 & 0.96 & 1.46 & 1.80 & 2.44 & 3.78 & \textbf{0.95} & 0.97 & 0.63 & 0.38 \\
CE-NSF & \textbf{0.60} & \textbf{0.56} & \textbf{0.61} & \textbf{0.74} & 1.58 & 1.30 & 1.50 & 2.31 & \textbf{0.95} & 0.93 & 0.90 & 0.76 \\
\addlinespace[4pt]
\multicolumn{13}{l}{\textit{Lotka--Volterra}} \\
Free-G & 0.99 & 0.99 & 1.00 & 1.00 & 23.60 & 34.30 & 52.05 & 75.06 & 0.94 & \textbf{0.96} & 0.70 & 0.26 \\
Scale-G & 0.99 & 0.99 & 1.00 & 1.00 & 24.42 & 25.25 & 24.07 & 24.11 & \textbf{0.95} & 0.87 & 0.63 & 0.33 \\
CE-G & 0.74 & 0.71 & 0.67 & \textbf{0.74} & \textbf{1.17} & \textbf{1.11} & \textbf{0.99} & 0.90 & 0.94 & 0.93 & 0.89 & \textbf{0.79} \\
Free-NSF & 0.97 & 0.98 & 0.99 & 1.00 & 2.93 & 3.84 & 5.43 & 7.45 & \textbf{0.95} & 0.97 & \textbf{0.90} & 0.56 \\
CE-NSF & \textbf{0.71} & \textbf{0.66} & \textbf{0.66} & 0.76 & 1.57 & 1.43 & 1.19 & \textbf{1.03} & 0.97 & \textbf{0.94} & 0.84 & 0.66 \\
\bottomrule
\end{tabular}
}
\end{table}

\clearpage
\section{Nonregular stress tests}
\label{app:boundary}

\paragraph{The four prototypes.}
Each prototype is a one- or two-dimensional model with an exact reference posterior, each chosen to probe one assumption of the structured representation:
\begin{itemize}
\item \textbf{T1, parameter-dependent support.}  $x_i\sim U[0,\theta]$, $\theta\sim U[1,3]$.  The posterior is $\propto\theta^{-n}$ on $[\max(M_n,1),3]$ with $M_n=\max_i x_i$ and contracts at rate $n^{-1}$ to a Gamma limit, the nonregular Bernstein--von Mises case of \citet{Bochkina2014}.  \emph{Violated: the imposed rate $1/2$.}
\item \textbf{T2, boundary parameter.}  $x_i\sim\cN(\theta,1)$ with $\theta\in[0,3]$ and a truncated-exponential prior of rate 2.  The posterior is $\mathrm{TN}(\bar x-2/n,\,1/n,\,[0,3])$: the rate $1/2$ holds, but the shape is non-Gaussian whenever the mass sits against the boundary at $0$.  \emph{Probed: residual shape at a boundary, which the Gaussian head can represent after renormalization.}
\item \textbf{T3, singular parameterization.}  $x_i\sim\cN(\theta_1\theta_2,1)$, $\theta\sim U[-2,2]^2$.  The posterior concentrates on the hyperbola $\theta_1\theta_2=\bar x$, a ridge that does not contract.  \emph{Violated: a single local center.}
\item \textbf{T4, persistent bimodality.}  $x_i\sim\cN(\theta^2,1)$, $\theta\sim U[-2,2]$.  The posterior has two modes at $\pm\sqrt{\bar x}$ whose separation does not shrink with $n$.  \emph{Violated: a single local center.}
\end{itemize}
Reference posteriors are computed on composite grids covering the prior support, and every estimator's density is renormalized on that support before scoring.  The two symmetric \texttt{sbibm} benchmarks belong to the same family: SLCP has four mirror-image modes (as T4), and the support boundary of two moons contracts at rate $n^{-1}$ (as T1).

\paragraph{Protocol.}
The protocol is that of the main text (Appendix~\ref{app:details}) with three differences: ten seeds instead of three, a smaller network (two hidden layers of 64 units, 3000 steps), and a nominal coverage level of $90\%$ with C2ST evaluated in closed form on the reference grid.  The structured estimators impose the same rate $\alpha=1/2$ as on the regular benchmarks, and the center-estimated estimators use the closed-form center of each prototype ($M_n$ for T1, $\max(\bar x,0)$ for T2, the nearest point on the ridge for T3, the positive mode for T4); for T1 the encoder additionally receives $M_n$ as a fixed summary, since mean pooling cannot represent a maximum.  The true rate ($1$ on T1) is used only in explicitly labeled oracle ablations.  One prototype-specific quantity is added to the main-text metrics: for T3 and T4, the distance from the posterior center to the nearest ridge or mode.

\begin{table}[htbp]
\centering
\caption{\textbf{Nonregular prototypes at $n=100$} ($5\times$ the largest
training size; ten seeds).  Each block violates a different assumption of
the structured representation.}
\label{tab:main}
\footnotesize
\begin{tabular}{l rrrr rrrr}
\toprule
 & \multicolumn{4}{c}{\emph{T1: parameter-dependent support}} & \multicolumn{4}{c}{\emph{T2: boundary parameter}} \\
\cmidrule(lr){2-5}\cmidrule(lr){6-9}
method & KL & SD ratio & center err. & C2ST & KL & SD ratio & center err. & C2ST \\
\midrule
Free-G & 12.27 & \textbf{1.09} & 4.72 & 0.92 & 0.12 & \textbf{1.00} & 0.59 & 0.61 \\
Scale-G & 1.31 & 2.69 & 2.02 & 0.82 & 0.24 & 0.83 & 0.39 & 0.61 \\
CE-G & 1.11 & 2.71 & \textbf{0.54} & 0.78 & \textbf{0.10} & 0.81 & \textbf{0.18} & \textbf{0.58} \\
Free-NSF & 0.65 & 2.02 & 0.87 & 0.70 & 0.17 & 1.02 & 0.52 & 0.61 \\
CE-NSF & \textbf{0.49} & 2.82 & 1.81 & \textbf{0.68} & 0.18 & 0.84 & 0.24 & 0.59 \\
\midrule
 & \multicolumn{4}{c}{\emph{T3: singular parameterization}} & \multicolumn{4}{c}{\emph{T4: persistent bimodality}} \\
\cmidrule(lr){2-5}\cmidrule(lr){6-9}
method & KL & SD ratio & center err. & C2ST & KL & SD ratio & center err. & C2ST \\
\midrule
Free-G & 2.34 & 0.84 & \textbf{0.01} & 0.90 & 2.28 & 0.85 & 0.01 & 0.86 \\
Scale-G & 7.65 & 0.36 & \textbf{0.01} & 0.95 & 5.01 & 0.39 & \textbf{0.00} & 0.93 \\
CE-G & 8.40 & 0.36 & 0.40 & 0.92 & 5.67 & 0.39 & 0.43 & 0.90 \\
Free-NSF & \textbf{0.67} & \textbf{0.95} & 0.11 & \textbf{0.70} & \textbf{0.14} & \textbf{1.03} & 0.01 & \textbf{0.61} \\
CE-NSF & 26.81 & 0.36 & 0.50 & 0.89 & 23.48 & 0.39 & 0.60 & 0.79 \\
\bottomrule
\end{tabular}

\end{table}

\paragraph{Rate violation: T1.}
T1 contracts at rate $n^{-1}$ while every estimator imposes $n^{-1/2}$.  The estimated posterior is therefore too wide, by a factor that grows with $n$, consistent with the rate sensitivity of Proposition~\ref{prop:horizon}.  Table~\ref{tab:b1-rate} then changes only the imposed exponent: the default $1/2$, the value $0.8$ that the training-range estimate of Appendix~\ref{app:alpha-hat} supplies, and the oracle $1$.  Width, center error and KL improve smoothly as the exponent approaches $1$, while coverage stays close to nominal for the default and estimated exponents, because a too wide posterior fails gracefully, and falls slightly below it at the oracle exponent.  An estimated rate therefore repairs most of the failure without being exact.

\begin{table}[H]
\centering
\caption{Result on T1 at $n=100$.}
\label{tab:b1-rate}
\resizebox{\linewidth}{!}{\begin{tabular}{l rrrrr rrrrr}
\toprule
 & \multicolumn{5}{c}{CE-G} & \multicolumn{5}{c}{CE-NSF} \\
\cmidrule(lr){2-6}\cmidrule(lr){7-11}
imposed $\alpha$ & SD ratio & center err. & cov.\ 90\% & C2ST & KL & SD ratio & center err. & cov.\ 90\% & C2ST & KL \\
\midrule
$1/2$ (default) & 2.71 & 0.54 & 0.97 & 0.78 & 1.10 & 2.82 & 1.81 & 0.97 & 0.68 & 0.49 \\
$0.8$ (estimated) & 1.40 & 0.25 & \textbf{0.91} & 0.72 & 0.67 & 1.42 & 0.59 & \textbf{0.92} & 0.59 & 0.14 \\
$1$ (oracle ablation) & \textbf{0.92} & \textbf{0.15} & 0.86 & \textbf{0.69} & \textbf{0.48} & \textbf{0.92} & \textbf{0.23} & 0.85 & \textbf{0.56} & \textbf{0.03} \\
\bottomrule
\end{tabular}
}
\end{table}

\paragraph{Boundary shape: T2.}
The fixed-$\theta_0$ strata of Table~\ref{tab:strata} isolate the boundary; away from it ($\theta_0=0.5$) every estimator is accurate, and at it neither residual learns the truncated shape. This is a learning failure, not a representational one: every estimator's density is renormalized on $[0,3]$ before scoring, and a Gaussian with mean $\bar x-2/n$ and variance $1/n$ so renormalized equals the true posterior exactly. The exact posterior's own $90\%$ HPD coverage of $\theta_0=0$ is $0.96$ at $n=20$ and $0.94$ at $n=100$, not the nominal $0.90$, because the truth sits on the boundary.

\begin{table}[H]
\centering
\caption{Fixed-$\theta_0$ strata, $n=20\to100$ (ten seeds).  T2 at the boundary ($\theta_0=0$), with the exact posterior as reference; T3/T4, distance from each estimator's posterior mean to the nearest ridge or mode in units of its width, which grows with $n$ for any single-center summary, the exact posterior mean included.  Accuracy on T3 and T4 is read from Table~\ref{tab:main}.}
\label{tab:strata}
\resizebox{\linewidth}{!}{\begin{tabular}{l rrrr rr}
\toprule
 & \multicolumn{4}{c}{T2, $\theta_0=0$ stratum} & T3 & T4 \\
\cmidrule(lr){2-5}\cmidrule(lr){6-6}\cmidrule(lr){7-7}
estimator & SD ratio & center err. & cov.\ 90\% & KL & ridge dist. & mode dist. \\
\midrule
Exact posterior & 1.00$\to$1.00 & 0.00$\to$0.00 & 0.96$\to$0.94 & 0.00$\to$0.00 & -- & -- \\
Free-G & 0.85$\to$0.74 & 0.19$\to$0.29 & 0.70$\to$0.93 & 0.11$\to$0.11 & 3.4$\to$7.9 & 9.5$\to$21.2 \\
Scale-G & 0.83$\to$0.86 & 0.19$\to$0.61 & 0.53$\to$0.03 & 0.13$\to$0.43 & 3.4$\to$7.9 & 9.5$\to$21.2 \\
CE-G & 0.83$\to$0.79 & 0.19$\to$0.22 & 0.60$\to$0.58 & 0.11$\to$0.16 & 3.4$\to$6.3 & 9.2$\to$12.7 \\
Free-NSF & 1.01$\to$1.14 & 0.03$\to$0.58 & 0.77$\to$0.26 & 0.01$\to$0.23 & 3.4$\to$7.8 & 9.5$\to$21.0 \\
CE-NSF & 0.97$\to$0.90 & 0.08$\to$0.18 & 0.68$\to$0.68 & 0.02$\to$0.07 & 3.3$\to$5.3 & 7.7$\to$8.1 \\
\bottomrule
\end{tabular}
}
\end{table}

\begin{table}[H]
\vspace{-10pt}
\centering
\caption{\textbf{The two symmetric multimodal \texttt{sbibm} benchmarks.}}
\label{tab:sbibm-symmetric}
\setlength{\tabcolsep}{3pt}
\renewcommand{\arraystretch}{1.0}
\resizebox{\linewidth}{!}{\begin{tabular}{l rrrr rrrr rrrr}
\toprule
 & \multicolumn{4}{c}{C2ST} & \multicolumn{4}{c}{MMD} & \multicolumn{4}{c}{$95\%$ joint coverage} \\
\cmidrule(lr){2-5}\cmidrule(lr){6-9}\cmidrule(lr){10-13}
estimator & $n=10$ & $n=20$ & $n=50$ & $n=100$ & $n=10$ & $n=20$ & $n=50$ & $n=100$ & $n=10$ & $n=20$ & $n=50$ & $n=100$ \\
\midrule
\multicolumn{13}{l}{\textit{SLCP}} \\
Free-G & 0.98 & 0.98 & 0.99 & 0.99 & 0.064 & 0.073 & 0.084 & 0.093 & 0.96 & 0.99 & \textbf{0.94} & \textbf{0.87} \\
Scale-G & 0.98 & 0.98 & 1.00 & 1.00 & 0.058 & 0.142 & 0.363 & 0.532 & 0.97 & 0.92 & 0.20 & 0.00 \\
CE-G & 0.98 & 0.98 & 0.98 & 0.99 & 0.040 & 0.072 & 0.144 & 0.208 & \textbf{0.95} & \textbf{0.94} & 0.87 & 0.69 \\
Free-NSF & \textbf{0.75} & \textbf{0.79} & \textbf{0.83} & \textbf{0.91} & \textbf{0.004} & \textbf{0.004} & \textbf{0.007} & \textbf{0.013} & 0.92 & 0.98 & \textbf{0.96} & 0.82 \\
CE-NSF & 0.91 & 0.97 & 0.98 & 0.98 & 0.147 & 0.290 & 0.427 & 0.514 & 0.27 & 0.15 & 0.13 & 0.09 \\
\addlinespace[4pt]
\multicolumn{13}{l}{\textit{Two moons}} \\
Free-G & 0.98 & 0.98 & 0.99 & 0.99 & 0.129 & 0.141 & 0.165 & 0.197 & \textbf{0.96} & 0.99 & \textbf{0.98} & \textbf{0.77} \\
Scale-G & 0.98 & 0.98 & 0.99 & 0.99 & 0.109 & 0.223 & 0.475 & 0.660 & 0.97 & 0.97 & 0.81 & 0.09 \\
CE-G & 0.97 & 0.98 & 0.99 & 0.99 & 0.112 & 0.195 & 0.347 & 0.432 & 0.98 & \textbf{0.95} & 0.57 & 0.39 \\
Free-NSF & \textbf{0.80} & \textbf{0.81} & \textbf{0.90} & \textbf{0.97} & \textbf{0.014} & 0.012 & \textbf{0.018} & \textbf{0.050} & \textbf{0.96} & 0.99 & 0.89 & 0.41 \\
CE-NSF & \textbf{0.80} & \textbf{0.81} & 0.94 & 0.98 & 0.016 & \textbf{0.011} & 0.119 & 0.221 & \textbf{0.96} & 0.96 & 0.40 & 0.11 \\
\bottomrule
\end{tabular}
}
\end{table}

\paragraph{No single local center: T3, T4, SLCP and two moons.}
On T3 and T4 every single-center estimator degrades with $n$; the free spline estimator degrades least (C2ST $0.70$ and $0.61$ at $n=100$ against $0.79$--$0.95$ for the others, Table~\ref{tab:main}) but is not accurate either.  On the two symmetric \texttt{sbibm} benchmarks (Table~\ref{tab:sbibm-symmetric}) every single-center estimator fails at every size, including inside the training range; the free spline estimator retains far more of the structure (MMD an order of magnitude lower) yet its C2ST at $n=100$ is still $0.91$ on SLCP and $0.97$ on two moons, so it does not recover the posterior either.

\clearpage
\section{Mechanism ablations and matched controls}
\label{app:diagnostics}
\label{app:revision-results}

Each subsection answers one question about the main-text comparison: whether the free baselines are under-trained, whether the ordering depends on the training maximum, whether the imposed rate matters once the center is held fixed, whether the exponent can be estimated from the training range, whether the residual network needs the size input, and how a compositional estimator compares.

\subsection{Best-effort free baselines}
\label{app:free-best}

The free baselines of the main text share the training budget of every other estimator.  To check whether they lose because they are under-trained rather than because they cannot extrapolate, Free-G and Free-NSF were retrained on all six regular benchmarks under a best-effort protocol: the same architecture, prior-standardized parameters, the learning rate selected from $\{10^{-3},3\times10^{-4},10^{-4}\}$ by in-range validation NLL, three times the step budget and a longer early-stopping patience.  No out-of-range data enter training, selection or calibration.  Table~\ref{tab:free-best} places the result beside the standard estimators at the training boundary and at $5\times$.

Within the training range, the extra effort narrows the gap on four benchmarks: on both Gaussian linear tasks, the Bernoulli GLM and SIR the best-effort baselines fit markedly better at the boundary, and on Gaussian linear with a uniform prior the best-effort Free-NSF matches the CE estimators there. On the Gaussian mixture and Lotka--Volterra, by contrast, the large in-range errors remain. Beyond the range the picture is the same on every benchmark: the best-effort baselines still widen and drift at $5\times$, their C2ST never falls below $0.83$ while CE-NSF stays at or below $0.76$, and on Gaussian linear the tuned Free-G loses most of its joint coverage. Extra training effort therefore narrows the in-range gap but does not change the out-of-range ordering.

\begin{table}[htbp]
\centering
\caption{\textbf{Best-effort free baselines beside the standard estimators} at the training boundary ($n=20$) and at $5\times$ ($n=100$); seed means over three seeds. Bold marks the best entry in each column and block (closest to the target where the metric has one).}
\label{tab:free-best}
\footnotesize
\resizebox{\linewidth}{!}{\begin{tabular}{lrrrrrrrr}
\toprule
 & \multicolumn{4}{c}{$n=20$ (training boundary)} & \multicolumn{4}{c}{$n=100$ ($5\times$)} \\
\cmidrule(lr){2-5}\cmidrule(lr){6-9}
estimator & SD ratio & center err. & cov. & C2ST & SD ratio & center err. & cov. & C2ST \\
\midrule
\multicolumn{9}{l}{\emph{Gaussian linear}} \\
Free-G & 1.09 & 0.11 & 0.97 & 0.61 & 1.41 & 0.59 & 0.86 & 0.89 \\
Free-G (best effort) & 1.01 & 0.07 & \textbf{0.96} & 0.54 & \textbf{0.98} & 0.81 & 0.39 & 0.93 \\
Free-NSF & 1.06 & 0.11 & 0.98 & 0.58 & 1.53 & 0.53 & 1.00 & 0.87 \\
Free-NSF (best effort) & 1.03 & 0.08 & \textbf{0.96} & 0.54 & 1.42 & 0.44 & 1.00 & 0.83 \\
CE-G & \textbf{1.00} & 0.05 & \textbf{0.96} & \textbf{0.51} & 1.06 & 0.12 & 0.97 & 0.57 \\
CE-NSF & 1.01 & \textbf{0.04} & \textbf{0.96} & \textbf{0.51} & \textbf{0.98} & \textbf{0.10} & \textbf{0.93} & \textbf{0.56} \\
\midrule
\multicolumn{9}{l}{\emph{Gaussian linear (uniform)}} \\
Free-G & 1.97 & 1.03 & 0.98 & 0.88 & 3.62 & 2.42 & 0.55 & 0.98 \\
Free-G (best effort) & 1.05 & 0.13 & \textbf{0.96} & 0.64 & 1.13 & 0.78 & 0.61 & 0.92 \\
Free-NSF & 1.12 & 0.15 & 0.98 & 0.62 & 1.62 & 0.74 & 0.99 & 0.92 \\
Free-NSF (best effort) & 1.04 & 0.09 & \textbf{0.96} & \textbf{0.56} & 1.37 & 0.48 & 0.98 & 0.83 \\
CE-G & \textbf{1.01} & 0.10 & \textbf{0.94} & 0.61 & 1.04 & 0.10 & 0.97 & 0.57 \\
CE-NSF & 1.02 & \textbf{0.08} & \textbf{0.94} & 0.59 & \textbf{0.98} & \textbf{0.08} & \textbf{0.93} & \textbf{0.56} \\
\midrule
\multicolumn{9}{l}{\emph{Gaussian mixture}} \\
Free-G & 3.42 & 2.48 & 0.93 & 0.90 & 5.88 & 3.67 & 0.98 & 0.95 \\
Free-G (best effort) & 3.14 & 2.36 & 0.94 & 0.89 & 5.10 & 3.20 & 0.98 & 0.95 \\
Free-NSF & 3.76 & 2.48 & 0.94 & 0.89 & 5.41 & 4.60 & 0.86 & 0.95 \\
Free-NSF (best effort) & 3.75 & 2.34 & \textbf{0.95} & 0.89 & 6.82 & 3.68 & 0.98 & 0.96 \\
CE-G & \textbf{1.14} & 0.07 & 0.94 & 0.55 & 0.87 & 0.06 & 0.88 & 0.54 \\
CE-NSF & 1.20 & \textbf{0.06} & 0.94 & \textbf{0.52} & \textbf{1.04} & \textbf{0.05} & \textbf{0.93} & \textbf{0.51} \\
\midrule
\multicolumn{9}{l}{\emph{Bernoulli GLM}} \\
Free-G & 8.88 & 7.00 & 0.94 & 1.00 & 19.49 & 16 & 0.88 & 1.00 \\
Free-G (best effort) & 2.84 & 1.95 & \textbf{0.95} & 0.99 & 5.84 & 4.82 & 0.77 & 1.00 \\
Free-NSF & 2.97 & 1.94 & \textbf{0.95} & 0.99 & 6.02 & 4.70 & 0.91 & 1.00 \\
Free-NSF (best effort) & 1.71 & 0.88 & 0.96 & 0.92 & 3.17 & 3.14 & 0.87 & 1.00 \\
CE-G & \textbf{1.04} & \textbf{0.13} & 0.94 & 0.67 & 1.10 & 0.12 & 0.97 & 0.68 \\
CE-NSF & \textbf{1.04} & \textbf{0.13} & \textbf{0.95} & \textbf{0.64} & \textbf{1.08} & \textbf{0.08} & \textbf{0.95} & \textbf{0.65} \\
\midrule
\multicolumn{9}{l}{\emph{SIR}} \\
Free-G & 3.07 & 2.07 & \textbf{0.96} & 0.92 & 7.43 & 18 & 0.36 & 0.99 \\
Free-G (best effort) & 1.41 & 0.47 & 0.98 & 0.71 & 2.91 & 5.27 & 0.47 & 0.94 \\
Free-NSF & 1.80 & 0.70 & 0.97 & 0.79 & 3.78 & 8.07 & 0.38 & 0.96 \\
Free-NSF (best effort) & 1.32 & 0.37 & 0.98 & 0.68 & 2.71 & 4.80 & 0.44 & 0.94 \\
CE-G & \textbf{1.08} & 0.17 & 0.93 & 0.58 & \textbf{2.07} & 0.41 & \textbf{0.97} & 0.77 \\
CE-NSF & 1.30 & \textbf{0.16} & 0.93 & \textbf{0.56} & 2.31 & \textbf{0.33} & 0.76 & \textbf{0.74} \\
\midrule
\multicolumn{9}{l}{\emph{Lotka--Volterra}} \\
Free-G & 34.30 & 22 & \textbf{0.96} & 0.99 & 75.06 & 47 & 0.26 & 1.00 \\
Free-G (best effort) & 29.73 & 20 & 0.97 & 0.99 & 65.84 & 45 & \textbf{0.85} & 1.00 \\
Free-NSF & 3.84 & 1.83 & 0.97 & 0.98 & 7.45 & 9.39 & 0.56 & 1.00 \\
Free-NSF (best effort) & 2.14 & 1.10 & 0.97 & 0.95 & 3.70 & 3.05 & 0.64 & 0.99 \\
CE-G & \textbf{1.11} & \textbf{0.19} & 0.93 & 0.71 & 0.90 & \textbf{0.13} & 0.79 & \textbf{0.74} \\
CE-NSF & 1.43 & \textbf{0.19} & \textbf{0.94} & \textbf{0.66} & \textbf{1.03} & 0.19 & 0.66 & 0.76 \\
\bottomrule
\end{tabular}
}
\end{table}

\subsection{A second training maximum on SIR}
\label{app:sir-n80}

To check that the ordering of the main text is a statement about the relative size $n/N$ rather than about the particular boundary $N=20$, every estimator (and the two best-effort free baselines) was retrained on SIR with training sizes drawn uniformly from $\{20,\dots,80\}$ and paired evaluation at $n\in\{40,80,200,400\}$, with fresh reference posteriors at those sizes and three seeds.  Table~\ref{tab:sir-n80} reports both training maxima on the common $n/N$ axis.  This study reuses the CE arms of the development runs, whose center is the likelihood-based Fisher-scoring estimator, at both training maxima; their $N=20$ rows therefore differ slightly from the main-text SIR rows, which use the likelihood-free moment center of Appendix~\ref{app:benchmarks}. The study was not rerun with the moment center because the ordering at $N=20$ is the same under either center.

CE retains the best C2ST at $5N$ at both training maxima, but the absolute errors differ. At $N=80$, CE-G has a width ratio of $1.16$ and joint coverage of $0.86$ at $5N$, compared with $1.83$ and $0.96$ at $N=20$; CE-NSF undercovers at $5N$ in both settings. The best-effort free baselines improve in-range fitting but still deteriorate out of range. Increasing the training maximum therefore preserves the observed C2ST ordering without making either accuracy or coverage invariant to $N$.

\IfFileExists{tables/sir_n80.tex}{
\begin{table}[htbp]
\centering
\caption{\textbf{SIR at two training maxima} against the relative size $n/N$.}
\label{tab:sir-n80}
\footnotesize
\resizebox{\linewidth}{!}{\begin{tabular}{llrrrrrrrrrrrr}
\toprule
 & & \multicolumn{4}{c}{C2ST} & \multicolumn{4}{c}{SD ratio} & \multicolumn{4}{c}{$95\%$ joint coverage} \\
\cmidrule(lr){3-6}\cmidrule(lr){7-10}\cmidrule(lr){11-14}
$N$ & estimator & $0.5N$ & $1N$ & $2.5N$ & $5N$ & $0.5N$ & $1N$ & $2.5N$ & $5N$ & $0.5N$ & $1N$ & $2.5N$ & $5N$ \\
\midrule
20 & Free-G & 0.88 & 0.92 & 0.97 & 0.99 & 2.27 & 3.07 & 4.37 & 7.43 & 0.94 & \textbf{0.96} & 0.61 & 0.36 \\
20 & Scale-G & 0.88 & 0.90 & 0.96 & 0.99 & 2.30 & 2.60 & 3.09 & 4.46 & \textbf{0.95} & 0.92 & 0.46 & 0.15 \\
20 & CE-G & 0.57 & 0.55 & 0.64 & 0.76 & \textbf{1.02} & \textbf{1.00} & \textbf{1.14} & \textbf{1.83} & 0.94 & 0.92 & \textbf{0.95} & \textbf{0.96} \\
20 & Free-NSF & 0.74 & 0.79 & 0.92 & 0.96 & 1.46 & 1.80 & 2.44 & 3.78 & \textbf{0.95} & 0.97 & 0.63 & 0.38 \\
20 & CE-NSF & \textbf{0.55} & \textbf{0.53} & \textbf{0.61} & \textbf{0.75} & 1.47 & 1.26 & 1.77 & 2.40 & 0.94 & 0.93 & 0.93 & 0.70 \\
\midrule
80 & Free-G & 0.96 & 0.97 & 0.99 & 0.99 & 5.30 & 6.79 & 11.63 & 17.81 & \textbf{0.95} & 0.97 & 0.73 & 0.56 \\
80 & Free-G (best effort) & 0.78 & 0.82 & 0.92 & 0.96 & 1.40 & 1.84 & 3.05 & 5.07 & 0.94 & 0.97 & 0.73 & 0.58 \\
80 & Scale-G & 0.95 & 0.96 & 0.98 & 0.99 & 4.36 & 4.51 & 5.90 & 7.79 & \textbf{0.95} & 0.92 & 0.61 & 0.31 \\
80 & CE-G & 0.58 & 0.58 & \textbf{0.64} & \textbf{0.72} & \textbf{1.04} & \textbf{1.03} & \textbf{1.02} & \textbf{1.16} & 0.94 & \textbf{0.94} & \textbf{0.92} & \textbf{0.86} \\
80 & Free-NSF & 0.89 & 0.91 & 0.97 & 0.99 & 2.69 & 3.50 & 5.95 & 10.05 & 0.96 & 0.97 & 0.71 & 0.47 \\
80 & Free-NSF (best effort) & 0.76 & 0.80 & 0.91 & 0.96 & 1.43 & 1.79 & 2.69 & 4.20 & 0.94 & 0.97 & 0.72 & 0.52 \\
80 & CE-NSF & \textbf{0.54} & \textbf{0.55} & \textbf{0.64} & \textbf{0.72} & 1.60 & 1.47 & 1.94 & 2.59 & 0.94 & \textbf{0.94} & \textbf{0.92} & 0.74 \\
\bottomrule
\end{tabular}
}
\end{table}
}{}

\subsection{Does the imposed rate matter at a fixed center?}
\label{app:matched}

This study compares, on the Gaussian mixture, two fixed-center controls built on the same ten-step EM center, the full CE-G model, and a standardized, tuned Free-G baseline.  The two controls keep the posterior mean exactly at the EM estimate (residual mean constrained to zero, no correction network) and differ only in whether the covariance carries the imposed $n^{-1/2}$ factor; CE-G is the full method; Free-G (standardized, tuned) is the free baseline with prior-standardized parameters and a learning rate selected by in-range validation.  The study uses the main-text protocol (training sizes up to $20$, evaluation at $20$ and $100$; Table~\ref{tab:revision-gmm}).

With the center held fixed, imposing the contraction factor improves C2ST and KL at $5\times$ over letting the covariance scale freely, while joint coverage moves from $0.974$ to $0.909$ against a nominal $0.95$; the rate parameterization therefore matters on its own for agreement with the reference, not for coverage.  In this comparison the fixed-center, imposed-scale control has lower C2ST and KL at $5\times$ than the full CE-G model, so the study isolates a benefit of scaling at a fixed center but does not show an additional gain from the learned location adjustment: once the center is accurate and the rate is right, most of the gain is already obtained, which is the scope stated in Section~\ref{sec:discussion}.

\IfFileExists{tables/revision_gmm.tex}{
\begin{table}[H]
\centering\footnotesize
\caption{Fixed-center controls on the Gaussian mixture at the training boundary ($n=20$) and $5\times$ ($n=100$); means across three seeds.}
\label{tab:revision-gmm}
\resizebox{\linewidth}{!}{\begin{tabular}{lrrrrr}
\toprule
method & C2ST at $n=20$ & KL at $n=20$ & C2ST at $n=100$ & KL at $n=100$ & cov. at $n=100$ \\
\midrule
Free-G (standardized, tuned) & 0.898 & 2.317 & 0.942 & 2.980 & \textbf{0.966} \\
CE-G ($\log n$ residual) & \textbf{0.552} & \textbf{0.121} & 0.537 & 0.096 & 0.870 \\
Fixed EM center, free scale & 0.574 & 0.174 & 0.623 & 0.328 & 0.974 \\
Fixed EM center, imposed scale & 0.558 & 0.137 & \textbf{0.513} & \textbf{0.033} & 0.909 \\
\bottomrule
\end{tabular}
}
\end{table}
}{}

The same two controls were run on SIR and Lotka--Volterra with the simulator-based moment center, under the main-text protocol and on the paired cases of Appendix~\ref{app:regular} (Table~\ref{tab:fixed-center-sir-lv}, C2ST only).  Imposing the contraction factor at a fixed center is again what matters: with the covariance free, C2ST rises with sample size on both tasks, whereas with the factor imposed it stays close to CE-G at every size and is slightly below it at $5\times$ (SIR $0.750$ against $0.772$, Lotka--Volterra $0.704$ against $0.737$).  The learned location adjustment therefore adds nothing measurable on these two tasks either.  The spline head does: CE-NSF has the lowest C2ST at every in-range size and on SIR at every size, but at $5\times$ on Lotka--Volterra it falls behind the fixed-center Gaussian control, and on both tasks its joint coverage at $n=100$ is well below nominal (Appendix~\ref{app:regular}), so the in-range benefit of residual flexibility does not carry over intact to the largest extrapolation.

\IfFileExists{tables/fixed_center_sir_lv.tex}{
\begin{table}[H]
\centering\footnotesize
\caption{Fixed-center controls on SIR and Lotka--Volterra: C2ST at every evaluation size (training sizes $5$ to $20$); means across three seeds, lowest per column in bold.}
\label{tab:fixed-center-sir-lv}
\resizebox{\linewidth}{!}{\begin{tabular}{lrrrrrrrr}
\toprule
 & \multicolumn{4}{c}{SIR} & \multicolumn{4}{c}{Lotka--Volterra} \\
\cmidrule(lr){2-5}\cmidrule(lr){6-9}
estimator & $n=10$ & $n=20$ & $n=50$ & $n=100$ & $n=10$ & $n=20$ & $n=50$ & $n=100$ \\
\midrule
Fixed center, free scale & 0.653 & 0.642 & 0.789 & 0.897 & 0.768 & 0.806 & 0.886 & 0.940 \\
Fixed center, imposed scale & 0.634 & 0.572 & 0.647 & 0.750 & 0.729 & 0.704 & 0.656 & \textbf{0.704} \\
CE-G & 0.640 & 0.579 & 0.663 & 0.772 & 0.738 & 0.713 & 0.671 & 0.737 \\
CE-NSF & \textbf{0.602} & \textbf{0.559} & \textbf{0.614} & \textbf{0.740} & \textbf{0.707} & \textbf{0.665} & \textbf{0.655} & 0.757 \\
\bottomrule
\end{tabular}
}
\end{table}
}{}

\subsection{Can the exponent be estimated from the training range?}
\label{app:alpha-hat}

CE-NPE imposes the contraction rate.  When no asymptotic argument supplies it, the only data available at training time are simulations at the training sizes, and the natural estimate is the rate at which the statistical center itself contracts there: for each of $J=64$ parameters drawn from the prior and each training size $n\in\{5,8,12,16,20\}$ we simulate $R=32$ replicate datasets, compute the spread of $c_{\mathrm{stat}}$ across replicates, and fit $\log\mathrm{SD}=a-\hat\alpha\log n$ by least squares.  This needs no likelihood and no reference posterior.  Table~\ref{tab:alpha-hat} reports $\hat\alpha$ on every benchmark next to the rate CE-NPE imposes ($1/2$ throughout) and the true asymptotic rate.  

Three readings follow.  On the regular simulators other than the mixture, the estimate is close to the regular rate.  On T1 it gives $0.80$, which is the finite-$n$ contraction of the maximum over $n\in[5,20]$ rather than the asymptotic $1$; imposing $0.80$ is exactly the estimated-rate row of Table~\ref{tab:b1-rate}, which repairs most of the T1 failure, so the estimate is useful where the asymptotic rate is unknown and wrong.  On the Gaussian mixture, the estimate is $1.38$, the narrow mixture component is identified between $n=5$ and $n=8$, and the center spread contracts much faster than $n^{-1/2}$ in this transient regime before settling to the regular rate.  An estimator that trusted the training range would impose an exponent that over-contracts out of range.  A training-range estimate is therefore a check on the model-class rate, not a substitute for it: when the two disagree, the disagreement itself indicates that the training range is not yet in the regime whose rate one wants to extrapolate.

\begin{table}[htbp]
\centering
\caption{Contraction exponent estimated from the training range.}
\label{tab:alpha-hat}
\begin{tabular}{lccc}
\toprule
benchmark & imposed rate & true asymptotic rate & estimated $\alpha$ \\
\midrule
Gaussian linear & $1/2$ & $1/2$ & $0.50$ \\
Gaussian linear (uniform) & $1/2$ & $1/2$ & $0.48$ \\
Gaussian mixture & $1/2$ & $1/2$ & $1.38$ \\
Bernoulli GLM & $1/2$ & $1/2$ & $0.46$ \\
SIR & $1/2$ & $1/2$ & $0.56$ \\
Lotka--Volterra & $1/2$ & $1/2$ & $0.48$ \\
T1 (support) & $1/2$ & $1$ & $0.80$ \\
T2 (boundary) & $1/2$ & $1/2$ & $0.43$ \\
T3 (singular) & $1/2$ & $\text{ridge}$ & $0.44$ \\
T4 (bimodal) & $1/2$ & $1/2$ & $0.47$ \\
\bottomrule
\end{tabular}

\end{table}

\subsection{Does the residual network still need $n$?}
\label{app:residual-context}
\paragraph{Gamma model and protocol.}
Training sizes are $\{20,50,100,200,500,1000\}$ and evaluation sizes are $\{1000,2000,5000,10000\}$.  The encoder is a Deep Sets network over the log-transformed observations.  The structured center is the log of the maximum-likelihood (equivalently method-of-moments) rate estimate plus a learned correction applied in rescaled units, $c_n=\log\hat\theta_{\mathrm{MLE}}+n^{-1/2}\Delta_\phi(s_n,\log n)$; the residual density on $z_n$ is a one-dimensional rational-quadratic spline flow with linear tails \citep{durkan2019neural}.  Each diagnostic uses five seeds.

We hold the structured center and $\sqrt n$ transformation fixed and change only whether the residual network receives $\log n$, $1/n$, or no size input.  As Table~\ref{tab:n-context} shows, all three variants have a fitted exponent of $0.500$ and nearly identical SD ratios, coverage, center error, and Kolmogorov--Smirnov (KS) statistics for simulation-based calibration (SBC) \citep{talts2018validating} at $10\times$.  The $\log n$ context improves KL from $0.0023$ without context to $0.0018$; $1/n$ gives $0.0012$. These differences concern finite-sample density fitting, while the contraction exponents remain nearly identical.  Sample size is not unnecessary: $n$ still enters the analytic transformation and the structured center.  In this Gamma ablation, explicit neural context becomes secondary once the dominant contraction law is encoded.
\begin{table}[H]
\centering
\caption{Matched residual-context comparison at $n=10{,}000$ (means over five seeds; coverage is the nominal $90\%$ region).}
\label{tab:n-context}
\resizebox{\linewidth}{!}{\begin{tabular}{lrrrrrr}
\toprule
residual context & KL & SD ratio & 90\% cov. & center error & SBC KS & exponent \\
\midrule
$\log n$ & $0.0018$ & $0.991$ & $0.886$ & $0.061$ & $0.039$ & $0.500$ \\
$1/n$ & $0.0012$ & $0.998$ & $0.889$ & $0.059$ & $0.038$ & $0.500$ \\
none & $0.0023$ & $0.992$ & $0.888$ & $0.065$ & $0.041$ & $0.500$ \\
\bottomrule
\end{tabular}
}
\end{table}

\subsection{A score-based compositional comparator on the six regular benchmarks}
\label{app:pfnpse}
Compositional estimators handle a large $n$ by learning the posterior score for a block of observations and summing block scores at test time, so every $n$ is supported without any single estimator having to extrapolate.  The score network of PF-NPSE \citep{geffner2023compositional} conditions on a block of up to $20$ observations; at test time the $n$ observations are partitioned into blocks of $20$ plus a remainder, each block is scored by the same network given its block size, and the block scores are summed with the diffused prior score subtracted once per additional block.  For conditionally independent observations, as on all six benchmarks, the factorization of the posterior into block posteriors is exact at noise level zero, but the sum of diffused block scores is an approximation at positive noise levels \citep{geffner2023compositional,linhart2026diffusion}; at $n\le 20$ the dataset forms a single block and no composition takes place, so only the two out-of-range sizes exercise it.  We run the authors' public implementation under the main-text protocol: the same number of training datasets as one main-text estimator, each simulated with $20$ replicates and randomly masked to a block size in $\{5,\dots,20\}$ (so about $1.6$ times the replicates a main-text estimator sees), the analytic diffused prior of the authors' code, which ignores the mild prior truncation of SIR and Lotka--Volterra, in-range validation only, three seeds, and the first 200 of the paired evaluation cases at every size.  Two samplers draw from the composed score: the published annealed Langevin sampler and the GAUSS sampler of \citet{linhart2026diffusion}, which corrects the composition with a Gaussian approximation.  Neither provides a density, so coverage is not reported.  Table~\ref{tab:pfnpse} places both beside Free-NSF and CE-NSF.

\begin{table}[!t]
\centering\footnotesize
\caption{\textbf{Score-based compositional estimator on the six regular benchmarks.}  PF-NPSE (three seeds) with the published annealed Langevin sampler and with the GAUSS sampler of \citet{linhart2026diffusion}.  Coverage requires a density and is not available for PF-NPSE.}
\label{tab:pfnpse}
\setlength{\tabcolsep}{4pt}
\renewcommand{\arraystretch}{1.15}
\begin{tabular}{l rrrr rrrr}
\toprule
 & \multicolumn{4}{c}{C2ST} & \multicolumn{4}{c}{SD ratio} \\
\cmidrule(lr){2-5}\cmidrule(lr){6-9}
estimator & $n=10$ & $n=20$ & $n=50$ & $n=100$ & $n=10$ & $n=20$ & $n=50$ & $n=100$ \\
\midrule
\multicolumn{9}{l}{\textit{Gaussian linear}} \\
Free-NSF & 0.55 & 0.58 & 0.72 & 0.87 & \textbf{1.01} & 1.06 & 1.26 & 1.53 \\
PF-NPSE (Langevin) & 0.56 & 0.57 & 0.57 & 0.62 & 1.11 & 1.12 & 1.03 & \textbf{0.98} \\
PF-NPSE (GAUSS) & \textbf{0.51} & 0.52 & 0.55 & 0.59 & \textbf{0.99} & \textbf{0.99} & \textbf{1.00} & \textbf{0.98} \\
CE-NSF & \textbf{0.51} & \textbf{0.51} & \textbf{0.54} & \textbf{0.56} & 1.02 & \textbf{1.01} & 0.99 & \textbf{0.98} \\
\addlinespace[4pt]
\multicolumn{9}{l}{\textit{Gaussian linear (uniform)}} \\
Free-NSF & 0.58 & 0.62 & 0.78 & 0.92 & \textbf{1.03} & 1.12 & 1.35 & 1.62 \\
PF-NPSE (Langevin) & 0.60 & 0.61 & 0.64 & 0.72 & 1.10 & 1.14 & 1.06 & 1.04 \\
PF-NPSE (GAUSS) & \textbf{0.57} & \textbf{0.59} & 0.68 & 0.74 & 0.94 & 0.96 & \textbf{1.00} & 1.04 \\
CE-NSF & 0.59 & \textbf{0.59} & \textbf{0.56} & \textbf{0.56} & 1.04 & \textbf{1.02} & \textbf{1.00} & \textbf{0.98} \\
\addlinespace[4pt]
\multicolumn{9}{l}{\textit{Gaussian mixture}} \\
Free-NSF & 0.87 & 0.89 & 0.93 & 0.95 & 3.14 & 3.44 & 4.08 & 4.76 \\
PF-NPSE (Langevin) & 0.94 & 0.95 & 0.96 & 0.96 & 5.65 & 6.06 & 5.62 & 5.83 \\
PF-NPSE (GAUSS) & 0.93 & 0.94 & 0.95 & 0.96 & 4.24 & 4.27 & 5.19 & 6.10 \\
CE-NSF & \textbf{0.56} & \textbf{0.52} & \textbf{0.51} & \textbf{0.51} & \textbf{1.33} & \textbf{1.15} & \textbf{1.05} & \textbf{1.01} \\
\addlinespace[4pt]
\multicolumn{9}{l}{\textit{Bernoulli GLM}} \\
Free-NSF & 0.98 & 0.99 & 1.00 & 1.00 & 2.24 & 2.97 & 4.42 & 6.02 \\
PF-NPSE (Langevin) & 0.75 & 0.83 & 0.89 & 0.94 & 1.18 & 1.33 & 1.14 & 1.10 \\
PF-NPSE (GAUSS) & 0.71 & 0.79 & 0.88 & 0.93 & \textbf{1.02} & 1.15 & 1.09 & \textbf{1.07} \\
CE-NSF & \textbf{0.67} & \textbf{0.64} & \textbf{0.63} & \textbf{0.65} & 1.05 & \textbf{1.04} & \textbf{1.05} & 1.08 \\
\addlinespace[4pt]
\multicolumn{9}{l}{\textit{SIR}} \\
Free-NSF & 0.74 & 0.79 & 0.92 & 0.96 & 1.46 & 1.80 & 2.44 & 3.78 \\
PF-NPSE (Langevin) & 0.78 & 0.83 & 0.84 & 0.91 & 1.70 & 2.09 & 2.19 & 3.76 \\
PF-NPSE (GAUSS) & 0.65 & 0.68 & 0.75 & 0.81 & \textbf{0.98} & \textbf{1.09} & \textbf{1.32} & \textbf{1.53} \\
CE-NSF & \textbf{0.60} & \textbf{0.56} & \textbf{0.61} & \textbf{0.74} & 1.58 & 1.30 & 1.50 & 2.31 \\
\addlinespace[4pt]
\multicolumn{9}{l}{\textit{Lotka--Volterra}} \\
Free-NSF & 0.97 & 0.98 & 0.99 & 1.00 & 2.93 & 3.84 & 5.43 & 7.45 \\
PF-NPSE (Langevin) & 0.97 & 0.98 & 0.98 & 0.99 & 2.12 & 2.51 & 2.17 & 2.51 \\
PF-NPSE (GAUSS) & 0.94 & 0.95 & 0.97 & 0.98 & \textbf{1.54} & 1.73 & 1.85 & 2.12 \\
CE-NSF & \textbf{0.71} & \textbf{0.66} & \textbf{0.66} & \textbf{0.76} & 1.57 & \textbf{1.43} & \textbf{1.19} & \textbf{1.03} \\
\bottomrule
\end{tabular}

\end{table}

Composition extrapolates well where the block posterior is easy to learn.  On Gaussian linear, PF-NPSE keeps the correct width at every size and its C2ST at $n=100$ is close to that of CE-NSF and far better than that of Free-NSF.  On the uniform-prior variant and on SIR it lies between the free baseline and the center-estimated estimator.  On the Gaussian mixture, Bernoulli GLM, and Lotka--Volterra it fails at every size, inside the training range as well as outside, as the free baselines do.  The mechanism differs from that of free conditioning: on Bernoulli GLM the composed width stays within a third of the reference at every size while the center error, in posterior standard deviations, grows roughly as $\sqrt{n}$, so the composed posterior has approximately the right width around a center whose error does not contract.  The GAUSS sampler improves width and C2ST over the published sampler on most benchmarks and sizes, most clearly on SIR, without changing the ordering.  Among the estimators in Table~\ref{tab:pfnpse}, CE-NSF has the best C2ST at $n=100$ on all six benchmarks, and PF-NPSE is the closest alternative only on Gaussian linear.  Configuration and per-case results accompany the code.

\section{Real-data application: mutation rate and growth from 1000 Genomes}
\label{app:popgen}

Population genetics supplies a real-data instance of the paper's
setting in which the center is a classical moment construction rather
than one of ours, and in which neither a ground-truth parameter nor a
tractable reference posterior exists: in-model evaluation is
truth-based, and real-data evaluation uses only reference-free
criteria.  The five main-text estimators are compared under the
main-text protocol; nothing else changes.

\subsection{Task and center}
Under the neutral coalescent \citep{kingman1982coalescent} with
exponential population growth, genomic window $i$ yields the summary
triple $x_i=(S_i,\pi_i,\eta_i)$ for $k$ haplotypes: segregating
sites, mean pairwise differences, and folded singletons.  All three
means are linear in the per-kb scaled mutation rate $\theta$, while
growth raises the rare-variant fraction $\eta/S$ and lowers $\pi/S$,
so a growth rate is identified jointly with $\theta$ from the
moments.  The growth parameter is the raw rate divided by four, which
keeps the rescaled residual inside the spline head's core interval
(Appendix~\ref{app:details}).

\smallskip\noindent
\begin{longtable}{@{}lp{0.7\linewidth}@{}}
\textbf{Prior} & $\theta\sim U[0.2,\,6]$ per kb; growth parameter
$\sim U[0,\,2]$ (raw rate $U[0,8]$) \\[2pt]
\textbf{Simulator} & coalescent tree of $k=20$ under exponential
growth (time-changed merging); mutations Poisson on branches, class
counts $\xi_b$, then $S=\sum_b\xi_b$,
$\pi=\sum_b b(k-b)\xi_b/\binom k2$, $\eta=\xi_1+\xi_{k-1}$ \\[2pt]
\textbf{Dimensions} & $\theta\in\R^{2}$, $x_i\in\R^{3}$ \\[2pt]
\textbf{Center} & simulator method of moments in the Watterson
tradition \citep{watterson1975number}: the growth rate from the two
observed ratios $(\bar\pi/\bar S,\bar\eta/\bar S)$ against moment
curves precomputed once by simulation, then $\theta$ from the scale
$\bar S$; a grid lookup with quadratic refinement at query time,
likelihood-free and contracting at $n^{-1/2}$ \\[2pt]
\textbf{Reference posterior} & none: the likelihood is intractable \\[2pt]
\textbf{References} & \citet{kingman1982coalescent,watterson1975number} \\
\end{longtable}

\subsection{Real data}
Summaries are computed from 1000 Genomes phase-3 chromosome 22
\citep{auton2015global}: $k=20$ haplotypes (ten YRI individuals),
biallelic PASS SNVs, non-overlapping 10\,kb windows with strict-mask
accessibility at least $0.9$, giving 404 windows (150{,}223
segregating sites).  An evaluation dataset of size $n$ subsamples $n$
windows at least 60\,kb apart; 400 such overlapping datasets are drawn per size
with a fixed seed shared by every estimator.  The real ratios
($\bar\pi/\bar S=0.258$ against the constant-size $0.282$;
$\bar\eta/\bar S=0.355$ against $0.297$) show the rare-variant excess
of a growing population.  Window independence is approximate and the
neutral model is misspecified for real genomes; the criteria below do
not condition on the model being correct.

\subsection{Reference-free criteria}
With no ground truth, each (estimator, $n$) is scored over the 400
subsampled datasets by:
\begin{itemize}
\item \textbf{Consistency}: the fraction of datasets whose two marginal
central $95\%$ intervals both contain the full-data moment fit
($\theta=1.70$, raw growth $0.75$, from all 404 windows), and the rms
standardized distance to it (radius). Because the fit is estimated from windows that overlap the
evaluation subsets and uses the same moment functional as the CE center, this is a consistency
diagnostic that favours CE-NPE, not a test of nominal coverage.
\item \textbf{Predictive}: the energy distance \citep{szekely2013energy} between
posterior-predictive window summaries (one simulated window per posterior draw, draws clipped to
the prior box) and held-out real windows, on standardized summaries, over the first 64 datasets
per size; this criterion is independent of the anchor.
\item \textbf{Posterior readout}: the typical posterior of
$(\theta,\gamma)$ at each $n$ (Figure~\ref{fig:popgen}).
\end{itemize}

\subsection{Results}
\paragraph{In-model SBC and coverage give complementary diagnostics.}
Using simulation-based calibration \citep{talts2018validating}, we assess posterior ranks alongside joint coverage (Table~\ref{tab:popgen2-sim}). At $5\times$ the largest training size the free heads are over-wide or rank-miscalibrated, Scale-G's learned center has substantial location error, and CE-G and CE-NSF have the lowest SBC KS statistics although their joint coverage is below nominal; Free-NSF has a radius and coverage closer to their targets but a larger SBC KS statistic. No single diagnostic therefore establishes posterior calibration (Section~\ref{sec:hierarchy}).

\paragraph{On the real genome CE-NPE maintains consistency with the full-data moment fit.}
The consistency of CE-G and CE-NSF with the full-data fit is flat from $n=10$ to $n=100$ while every other estimator's decays, their standardized distance to the fit stays constant while the others' grows, and their predictive energy distance matches the best free baseline within seed variability (Table~\ref{tab:popgen2-real}). At $n=100$ their posteriors sit on the full-data fit with a positive growth rate, as expected for an expanding population, whereas Scale-G halves the growth rate and shifts $\theta$ by more than a posterior SD (Figure~\ref{fig:popgen}).

\begin{table}[H]
\centering
\caption{\textbf{In-model protocol, every size} (seed means): rms standardized truth distance (radius; target $1$), joint $95\%$ HPD coverage, and maximum marginal SBC KS; bold: within one seed SD of the best per column.}
\label{tab:popgen2-sim}
\footnotesize
\resizebox{\linewidth}{!}{\begin{tabular}{lrrrrrrrrrrrr}
\toprule
 & \multicolumn{3}{c}{$n=10$} & \multicolumn{3}{c}{$n=20$} & \multicolumn{3}{c}{$n=50$} & \multicolumn{3}{c}{$n=100$} \\
\cmidrule(lr){2-4}\cmidrule(lr){5-7}\cmidrule(lr){8-10}\cmidrule(lr){11-13}
 & rad. & cov. & KS & rad. & cov. & KS & rad. & cov. & KS & rad. & cov. & KS \\
\midrule
Free-G & 1.04 & 0.94 & 0.052 & 0.87 & 0.98 & 0.063 & 0.82 & 1.00 & 0.087 & 0.82 & \textbf{0.98} & 0.131 \\
Free-NSF & 0.98 & \textbf{0.95} & \textbf{0.033} & 0.92 & 0.98 & 0.062 & 0.92 & 0.97 & 0.110 & \textbf{1.01} & \textbf{0.92} & 0.224 \\
Scale-G & \textbf{1.00} & \textbf{0.95} & 0.042 & 1.06 & 0.96 & 0.077 & 1.34 & 0.87 & 0.167 & 1.64 & 0.75 & 0.181 \\
CE-G & 0.97 & \textbf{0.95} & 0.046 & \textbf{0.98} & \textbf{0.96} & \textbf{0.053} & 1.14 & \textbf{0.94} & 0.091 & 1.24 & 0.87 & \textbf{0.075} \\
CE-NSF & \textbf{1.00} & \textbf{0.95} & 0.046 & 0.95 & \textbf{0.95} & \textbf{0.048} & \textbf{1.04} & 0.91 & \textbf{0.039} & \textbf{1.10} & 0.86 & \textbf{0.068} \\
\bottomrule
\end{tabular}
}
\end{table}

\begin{table}[H]
\centering
\caption{\textbf{1000 Genomes windows, every size} (seed means): consistency with and standardized distance to the full-data moment fit, and posterior-predictive energy distance to held-out windows; bold: within one seed SD of the best per column.}
\label{tab:popgen2-real}
\footnotesize
\resizebox{\linewidth}{!}{\begin{tabular}{lrrrrrrrrrrrr}
\toprule
 & \multicolumn{3}{c}{$n=10$} & \multicolumn{3}{c}{$n=20$} & \multicolumn{3}{c}{$n=50$} & \multicolumn{3}{c}{$n=100$} \\
\cmidrule(lr){2-4}\cmidrule(lr){5-7}\cmidrule(lr){8-10}\cmidrule(lr){11-13}
 & cons. & rad. & energy & cons. & rad. & energy & cons. & rad. & energy & cons. & rad. & energy \\
\midrule
Free-G & 0.79 & 1.32 & 0.162 & 0.81 & \textbf{1.32} & 0.132 & 0.77 & 1.32 & 0.094 & 0.66 & 1.59 & 0.085 \\
Free-NSF & 0.80 & 1.28 & 0.158 & 0.79 & 1.52 & \textbf{0.123} & 0.72 & 1.63 & 0.090 & 0.58 & 1.92 & \textbf{0.077} \\
Scale-G & \textbf{0.83} & 1.21 & 0.158 & 0.77 & 1.44 & 0.129 & 0.58 & 1.87 & 0.095 & 0.28 & 2.79 & 0.085 \\
CE-G & \textbf{0.83} & 1.23 & \textbf{0.153} & 0.81 & \textbf{1.32} & \textbf{0.122} & \textbf{0.85} & \textbf{1.22} & \textbf{0.087} & \textbf{0.83} & \textbf{1.23} & 0.079 \\
CE-NSF & \textbf{0.83} & \textbf{1.18} & 0.160 & \textbf{0.83} & \textbf{1.28} & \textbf{0.121} & \textbf{0.86} & \textbf{1.22} & \textbf{0.088} & \textbf{0.85} & 1.28 & 0.082 \\
\bottomrule
\end{tabular}
}
\end{table}

\end{document}